\documentclass[twoside,11pt]{article}

\usepackage{blindtext}
\usepackage{graphicx}
\usepackage{authblk}
\usepackage{natbib}
\bibpunct{(}{)}{;}{a}{,}{,}

\usepackage[utf8]{inputenc}
\usepackage[T1]{fontenc}
\usepackage[hidelinks]{hyperref}
\usepackage{url}
\usepackage{booktabs}
\usepackage{amsfonts}
\usepackage{nicefrac}
\usepackage{microtype}
\usepackage{xcolor}

\usepackage{caption}
\usepackage{subcaption}

\usepackage{algorithm}
\usepackage{algpseudocode}
\usepackage{algorithmicx}

\usepackage{mathtools}
\usepackage{amsmath}
\usepackage{amssymb}
\usepackage{bm}
\usepackage{dsfont}
\usepackage[nameinlink,capitalize,noabbrev]{cleveref}

\newcommand{\BlackBox}{\rule{1.5ex}{1.5ex}}
\newenvironment{proof}{\par\noindent{\bf Proof\ }}{\hfill\BlackBox\\[2mm]}
 
\newtheorem{theorem}{Theorem}
\newtheorem{lemma}[theorem]{Lemma}

\newtheorem{definition}[theorem]{Definition}

\usepackage{gensymb} 

\usepackage{aligned-overset}

\graphicspath{ {figures/} }

\newcommand{\ie}{i.e.}
\newcommand{\eg}{e.g.}

\renewcommand{\vec}[1]{\bm{#1}} 
\newcommand{\mat}[1]{\bm{#1}} 
\newcommand{\set}[1]{\mathcal{#1}} 
\newcommand{\event}[1]{#1} 
\newcommand{\range}[1]{[#1]} 
\newcommand{\conv}{\circledast} 
\usepackage{arydshln} 

\newcommand{\cupdot}{\mathbin{\mathaccent\cdot\cup}}
\renewcommand{\uplus}{\cupdot}

\makeatletter
\providecommand*{\boxast}{%
  \mathbin{
    \mathpalette\@boxit{*}%
  }%
}
\newcommand*{\@boxit}[2]{%
  \sbox0{$\m@th#1\Box$}%
  \ifx#1\displaystyle \ht0=\dimexpr\ht0+.05ex\relax \fi
  \ifx#1\textstyle \ht0=\dimexpr\ht0+.05ex\relax \fi
  \ifx#1\scriptstyle \ht0=\dimexpr\ht0+.04ex\relax \fi
  \ifx#1\scriptscriptstyle \ht0=\dimexpr\ht0+.065ex\relax \fi
  \sbox2{$#1\vcenter{}$}
  \rlap{%
    \hbox to \wd0{%
      \hfill
      \raisebox{%
        \dimexpr.5\dimexpr\ht0+\dp0\relax-\ht2\relax
      }{$\m@th#1#2$}%
      \hfill
    }%
  }%
  \Box
}
\makeatother

\newcommand{\rconv}{\boxast} 

\newcommand{\R}{\mathbb{R}} 
\newcommand{\N}{\mathbb{N}} 
\newcommand{\B}{\mathbb{B}} 

\newcommand{\ind}[1]{\mathds{1}[#1]} 
\renewcommand{\Pr}{\mathbb{P}} 
\newcommand{\gw}{w} 

\newcommand{\Normal}{\mathcal{N}} 
\newcommand{\Unif}{\mathrm{Unif}} 

\DeclareMathOperator{\sign}{sign} 
\DeclarePairedDelimiterX{\inp}[2]{\langle}{\rangle}{#1, #2} 
\DeclarePairedDelimiterX{\dinp}[2]{\langle\langle}{\rangle\rangle}{#1, #2} 
\DeclarePairedDelimiter{\norm}{\lVert}{\rVert} 
\DeclarePairedDelimiter{\card}{\lvert}{\rvert} 
\DeclareMathOperator*{\argmin}{arg\,min} 
\DeclareMathOperator{\T}{T} 
\DeclareMathOperator{\OO}{\mathcal{O}} 
\DeclareMathOperator{\logOrd}{\widetilde{\mathcal{O}}} 
\DeclareMathOperator*{\E}{\mathbb{E}} 

\newcommand{\dist}{\mathrm{d}} 
\newcommand{\shift}{\mathrm{S}} 
\newcommand{\lift}[1]{{#1}^\uparrow} 
\DeclareMathOperator*{\GMP}{GMP} 

\newcommand{\phihat}{\hat{\varphi}}
\newcommand{\radius}{R_{\operatorname{rad}}}

\begin{document}

\title{\Large\bf Memorization Capacity of Invariant Convolutional Neural Nets}
\date{}
\author{Sjoerd Dirksen$^*$, Patrick Finke$^*$, and Paul Geuchen\thanks{Mathematical Institute, Utrecht University, The Netherlands (e-mails: \texttt{[\href{mailto:s.dirksen@uu.nl}{s.dirksen}, \href{mailto:p.geuchen@uu.nl}{p.geuchen}]@uu.nl}, \texttt{\href{mailto:patrick_finke@gmx.de}{patrick\_finke@gmx.de}})}}

\maketitle

\begin{abstract}%
    \noindent Neural networks are powerful memorizers which are often able to interpolate their training data even when the labels have been randomized. While there exists an extensive literature on the memorization capacity of fully connected networks, in practice more specific architectures with problem-dependent properties are preferred, such as convolutional networks for image classification. In this paper we consider convolutional networks which are invariant with respect to a given arbitrary finite group of orthogonal transformations, which may include translations, reflections, and/or rotations of images and invariances of higher order tensors. We prove upper bounds on the size of an invariant convolutional neural network with threshold activations that is required to interpolate a given dataset with two classes. Our bounds are formulated in terms of a complexity measure that takes the invariant structure into account and is instance-specific, \ie, it is related to the mutual arrangement of the classes in the given dataset and their geometric properties rather than the number of samples. Our proof is constructive: for given input data, instances of an interpolating network can be computed via a randomized algorithm with high probability. In numerical experiments we investigate the practical effectiveness of this algorithm.

\end{abstract}

\section{Introduction}
\label{sec:introduction}

The \emph{memorization capacity} of a neural network architecture quantifies the number of parameters and neurons needed for the architecture to be able to interpolate \emph{any} $N$ data points for \emph{any} assignment of labels (for a suitable setting of the weights). The study of this capacity has a long history---dating back all the way to the 1960s---and its importance experienced a recent revival due to the \emph{double descent phenomenon} \citep{belkin2019}. According to the long-standing \emph{bias-variance trade-off} \citep{shalev2014understanding, hastie2009elements}, the complexity of a machine learning model needs to balance underfitting and overfitting to achieve good generalization performance. It should be rich enough to express the underlying patterns in the data but still be simple enough to avoid picking up any spurious patterns. This classical viewpoint is at odds with modern practice in deep learning, where highly over-parameterized neural networks are trained to achieve a (nearly) perfect fit on the training data \citep[see, e.g.,][]{canziani2016analysis,zhang2017understanding}. In spite of this apparent overfitting, however, these models perform well on test data. This behavior is described by the double descent curve, where increasing the models complexity beyond the \emph{interpolation threshold}---where the model is just complex enough to interpolate its training data---results in a second descent of the test error. This behavior is well documented for deep neural networks as well as a wide variety of other machine learning models, see, \eg, \cite{belkin2019,belkin2021, nakkiran2019deep,mei2020generalization, hastie2020surprises}. As was argued by \cite{dirksen2024}, in practical applications the model size is tailored to a given data set. Memorization capacity, however, studies the reverse situation, requiring a model of fixed size to interpolate \emph{any} dataset of $N$ samples. This results in bounds that necessarily grow with the number of samples. However, for a problem with a `nice' structure, such as linearly separable classes in the simplest case, this should certainly not be necessary.

Present works on memorization capacity typically only consider fully-connected neural networks. However, these are only rarely used in practice and more specific architectures are preferred. A prominent example are \emph{convolutional neural networks} for the processing of sequential or image data. Additionally, those models are frequently endowed with special problem-dependent properties such as \emph{translation-invariance}, \eg, through data augmentation in the training process or the use of global pooling operations, see, \eg, \citet{biscione2021,blything2021human}. This leads to the question:
\begin{center}
    \emph{What is the memorization capacity of a translation-invariant neural architecture?}
\end{center}
We recall that memorization capacity requires an architecture to be able to interpolate any $N$ samples for any assignment of labels. When two samples are the same up to translation, interpolating any label assignment simply becomes impossible. At the very least, one would need to identify these samples and assign them the same label. At the same time, instead of the total number of samples one should then consider only the number of unique samples up to translation, as any translated copies are correctly classified automatically.

In this paper, we remedy the above issues via an instance-specific approach which additionally respects the translation-invariant property of the model. Specifically, we build upon the prior work of \cite{dirksen2024} who derived an instance-specific result for fully-connected networks. Our main contributions are the following:
\begin{enumerate}
    \item \textbf{Memorization capacity of translation-invariant neural nets.} In Section~\ref{sec:translation-invariant-memorization}, given a binary classification data set, we prove upper bounds on the required number of parameters and neurons for a translation-invariant convolutional neural network to be able to interpolate the data. The bounds are linked to a new complexity measure that respects both the translation-invariant property of the model, as well as the mutual arrangement of the two classes and their geometric properties.
    \item \textbf{Randomized interpolation algorithm for convolutional neural nets.} Our proof is constructive: we introduce a randomized algorithm that, given a binary classification data set, with high probability, outputs a translation-invariant convolutional neural network which interpolates the data and has at most the asserted size. In numerical experiments, we investigate the effectiveness of our algorithm both in controlled experiments and on real image data, and compare its performance to previous work.
    \item \textbf{Beyond just translations.} In Appendix~\ref{sec:group-invariant-memorization}, we prove our results in a more general framework, which supports invariance with respect to a given, arbitrary finite group of orthogonal transformations. Apart from translations, this also includes, for example, reflections and rotations of image data, and invariances of higher order tensor data. We have chosen to present our results first in the special setting of translation-invariant neural nets to keep the exposition accessible.
\end{enumerate}
We develop our results for networks with threshold activations, as this is a challenging setting in memorization theory. We expect that our results can be extended to other activations that zero out negative inputs and assign positive values to positive inputs, such as the ReLU.

\paragraph{Notation.}

We denote scalars $c, C \in \R$ with plain letters. We denote vectors $\vec{x} \in \R^d$ with bold lowercase letters, and their $i$-th entries by $x_i$ or $[x]_i$. We denote matrices $\mat{A} \in \R^{m \times n}$ with bold uppercase letters, their $i$-th rows by $\vec{A}_i$ or $[\vec{A}]_i$, and their $(i, j)$-th entries by $A_{i, j}$ or $[A]_{i, j}$. $\B_2^d$ denotes the unit Euclidean ball in $\R^d$. For two disjoints sets $\set{A}$ and $\set{B}$ we denote their union by $\set{A} \uplus \set{B}$. For $n \in \N$, denote $\range{n} \coloneqq \{ 1, \ldots, n \}$. We write $A \gtrsim B$, or alternatively $B = \OO(A)$, if $C \cdot A \geq B$ for an absolute constant $C > 0$. When we do not only hide constants but logarithmic terms as well, we write $B = \logOrd(A)$ instead. We denote the \emph{threshold activation} and \emph{sign function} by
\begin{equation*}
    \sigma(t) = \begin{cases}
        1, & t > 0,\\
        0, & \text{else},
    \end{cases}
    \quad\text{and}\quad
    \sign(t) = \begin{cases}
        +1, & t > 0,\\
        -1, & \text{else},
    \end{cases}
\end{equation*}
respectively. For $\vec{a}, \vec{b} \in \R^d$, we define $\ind{\vec{a} = \vec{b}} \in \{0, 1\}^d$ element-wise by
\begin{equation*}
    (\ind{\vec{a} = \vec{b}})_i = \begin{cases}
        1, & a_i = b_i,\\
        0, & \text{else}.
    \end{cases}
\end{equation*}
Moreover, we define the \emph{global max pooling} operation by $\GMP(\vec{x}) = \max_{i \in \range{d}} x_i$. We denote by $\vec{0}, \vec{1} \in \R^d$ the vectors with entries all equal to $0$ and $1$, respectively. We denote the standard multivariate \emph{normal distribution} in $d$ dimensions by $\Normal(\vec{0}, \mat{I}_d)$ and the \emph{uniform distribution} on $\set{A} \subset \R^d$ by $\Unif(\set{A})$.

\subsection{Related work}
\label{sec:related-work-short}

There is a rich literature on the memorization capacity of neural network architectures, especially for fully-connected architectures. If a model is able to interpolate any $N$ samples in $\R^d$ with any assignment of labels, we say that it can \emph{memorize} $N$ samples. Memorization capacity quantifies how large a model needs to be in order to memorize $N$ samples.

\paragraph{Fully-connected networks.}

Earlier works on the memorization capacity of neural networks with threshold activations date back to the 1960s and studied shallow architectures, see, \eg, \citet{cover1965,nilsson1965learning,minsky1969introduction,baum1988capabilities,huang1990bounds}. More recent results consider deeper networks. In particular, \citet{vershynin2020memory} assumed samples on the unit sphere with at least a Euclidean distance of $\delta$ between them, and proved that a deep threshold (or ReLU) network can memorize $N$ samples with binary labels using $\logOrd(e^{1/\delta^2} + \sqrt{N})$ neurons and $\logOrd(e^{1/\delta^2} (d + \sqrt{N}) + N)$ parameters. Later, \citet{rajput2021exponential} removed the exponential dependence on $\delta$, showing that only $\logOrd(1/\delta + \sqrt{N})$ neurons and $\logOrd(d / \delta + N)$ parameters are sufficient while further only requiring bounded norm instead of unit norm. Interestingly, while the earlier results are deterministic constructions, \citet{vershynin2020memory} and \citet{rajput2021exponential} used a probabilistic approach. Recently, \citet{dirksen2024} presented an instance-specific approach to memorization capacity. For a given finite binary classification dataset, their bound on the required number of neurons and parameters depends not on the number of samples but is instead linked to the geometric properties of the data classes and their mutual arrangement. For benign data arrangements, their result can overcome the worst-case estimates of memorization capacity.

Among other activations, the ReLU stands out in particular due to its widespread use in practice. Some earlier works include the one of \citet{bubeck2020network} who extended the result of \citet{baum1988capabilities} to real-valued labels, and \citet{zhang2017understanding} who used weight-sharing. More recently, \citet{park2021provable} used deep ReLU networks to achieve the first result with a sublinear number of parameters. Assuming that the points are separated, they showed that a ReLU (or hard-Tanh) network can memorize multi-class labels using a constant number of neurons, $\logOrd(N^{2/3})$ layers and $\logOrd(d + N^{2/3})$ parameters. \citet{vardi2021optimal} improved the above dependence on $N$. For separated samples with bounded norm they showed that a ReLU network with a constant number of neurons, $\logOrd(\sqrt{N})$ layers and $\logOrd(d + \sqrt{N})$ parameters can memorize $N$ samples with multi-class labels. \citet{lee2024defining} derived bounds on the required size of an interpolating fully-connected ReLU network which they link to the polytope structure of a given data set and further connect to other manifold shapes.

For results with activations other than threshold or ReLU see, \eg, \citet{huang1998upper,Hua03,park2021provable,MaT23}. The works we cite above make assumptions on the samples such as general position, separation, or bounded norm. For results for randomized samples or samples drawn from a random distribution, see, \eg, \citet{GWZ19,Dan20,ZZH21}.

\paragraph{Convolutional neural networks.}

Convolutional neural networks \citep{lecun1989handwritten,lecun1995convolutional} are a cornerstone of modern machine learning. They have brought breakthroughs in the processing of images, video, speech, and audio \citep{goodfellow2016,lecun2015deep} and have been a driving force for the resurgence of deep learning \citep{krizhevsky2012imagenet}. However, in contrast to fully-connected networks, their memorization capacity does not yet seem to be very well studied.

A few empirical investigations have addressed this topic. Notably, \citet{zhang2017understanding} observed that popular convolutional architectures can easily interpolate a random labeling of their training data. \citet{arpit2017closer} expanded on this observation and argued in particular that the network's capacity depends on the considered data set. On an image data set with randomized labels, \citet{sablayrolles2018dejavu} gradually increased the number of training samples for convolutional networks of different sizes and observed at which point the training accuracy drops.

While the universality of convolutional neural networks (with max pooling) was shown by \citet{cohen2016convolutional}, to our knowledge the only existing result on their memorization capacity is due to \citet{nguyen2018optimization}. They showed that a convolutional neural network can memorize $N$ samples with real-valued labels, assuming that the patches of the samples are distinct and that the convolutional layer before the fully-connected output layer has output dimension at least $N$. However, their results require an activation function which is at least continuous, unlike the threshold activation which we use here. Moreover, their network size depends on the number of samples, which is not the case in our result.

As a natural extension to convolutional neural networks, \citet{cohen2016group} introduced group equivariant neural networks (G-CNNs). They developed group convolutions which generalize standard convolutions to be equivariant under various group transformations, such as translations, rotations, and reflections. There is a growing body of work on G-CNNs, see, e.g., \citet{cohen2019general} and the references therein.
In Appendix~\ref{sec:group-invariant-memorization} we use group convolutional layers to prove our results in a generalized setting.

\paragraph{Other architectures.}

There is research on the memorization capacity of other architectures, such as residual networks \citep{hardt2017identity}, fully-connected and convolutional autoencoders \citep{radhakrishnan2019memorization}, networks with conditional computation \citep{koyuncu2023memorization}, and, more recently, transformer models \citep{kajitsuka2025on,madden2024nexttoken,mahdavi2024memorization}.

\section{Translation-invariant memorization algorithm}
\label{sec:translation-invariant-memorization}

In this section, we present our results on translation-invariant convolutional neural networks. These follow from a more general framework where invariance under an arbitrary group of orthogonal transformations can be captured, including reflections and rotations of image data. We refer to Appendix~\ref{sec:group-invariant-memorization} for the technical details. For simplicity, we focus on the binary classification setting. Our method is extendable to multiple classes via a one-vs-many approach, see Appendix~\ref{sec:multiclass-algorithm} for this straightforward extension.

In the following, let $\set{X} = \set{X}^+ \uplus \set{X}^- \subset \R^{L \times H \times W}$ denote a given finite data set, where $\set{X}^+$ and $\set{X}^-$ are disjoint sets containing the samples with label $+1$ and $-1$, respectively. We assume that these are images which are $H$ pixels high, $W$ pixels wide, and have $L$ color channels. Our goal is the construction of a translation-invariant convolutional neural network which interpolates this data set.

To state this goal formally, denote by $\T_{h, w}$ the circular translation in pixel-space by $h$ pixels vertically and $w$ pixels horizontally, see the examples in Figure~\ref{fig:translation-invariant-mutual-covering} for an illustration. Formally, for a single color channel $\vec{z} \in \R^{H \times W}$,
\begin{equation*}
    \T_{h, w} \colon \R^{H \times W} \to \R^{H \times W},
    \quad \vec{z} \mapsto [z_{(h' - h \bmod H, w' - w \bmod W)}]_{(h', w') \in \range{H} \times \range{W}},
\end{equation*}
and we can use the lifting  
\begin{equation}
    \label{eq:lifted-translation}
    \lift{\T}_{h, w} \colon \R^{L\times H \times W} \to \R^{L\times H \times W},
    \quad (\vec{z}_{\ell})_{\ell\in[L]} \mapsto (\T_{h, w}(\vec{z}_{\ell}))_{\ell\in[L]}
\end{equation}
to translate all color channels at once.

\begin{definition}
    We say that $F \colon \R^{L \times H \times W} \to \{ \pm 1 \}$ is \emph{translation-invariant} if
    \begin{equation*}
        F(\lift{\T}_{h, w} \vec{x}) = F(\vec{x})
    \end{equation*}
    for all $\vec{x} \in \R^{L \times H \times W}$ and $(h, w) \in \range{H} \times \range{W}$.
\end{definition}

In their instance-specific result on memorization capacity for fully-connected threshold neural networks, \citet{dirksen2024} utilize a \emph{mutual covering} of Euclidean balls to quantify the difficulty of interpolating a given data set. When a model is translation-invariant it identifies any sample and its translated copies to one another. Hence, as a natural extension of the Euclidean distance to our setting, we define the pseudo-distance
\begin{equation*}
    \dist(\vec{x}, \vec{y})
    \coloneqq \min_{(h, w) \in \range{H} \times \range{W}} \norm{\vec{x} - \T_{h, w} \vec{y}}_2.
\end{equation*}
It is easy to see that this pseudo-distance is translation-invariant and has the property that $\dist(\vec{x}, \vec{y}) = 0$ if and only if $\vec{x}$ and $\vec{y}$ are translates of each other.  
For sets $\set{A}, \set{B} \subset \R^{L \times H \times W}$ and $\vec{a} \in 
\R^{L \times H \times W}$, we define
\begin{equation*}
    \dist(\vec{a}, \set{B}) \coloneqq \inf_{\vec{b} \in \set{B}} \dist(\vec{a}, \vec{b})
    \quad\text{and}\quad
    \dist(\set{A}, \set{B}) \coloneqq \inf_{\vec{a} \in \set{A}} \dist(\vec{a}, \set{B}).
\end{equation*}

\begin{figure}
    \includegraphics[width=0.6\linewidth]{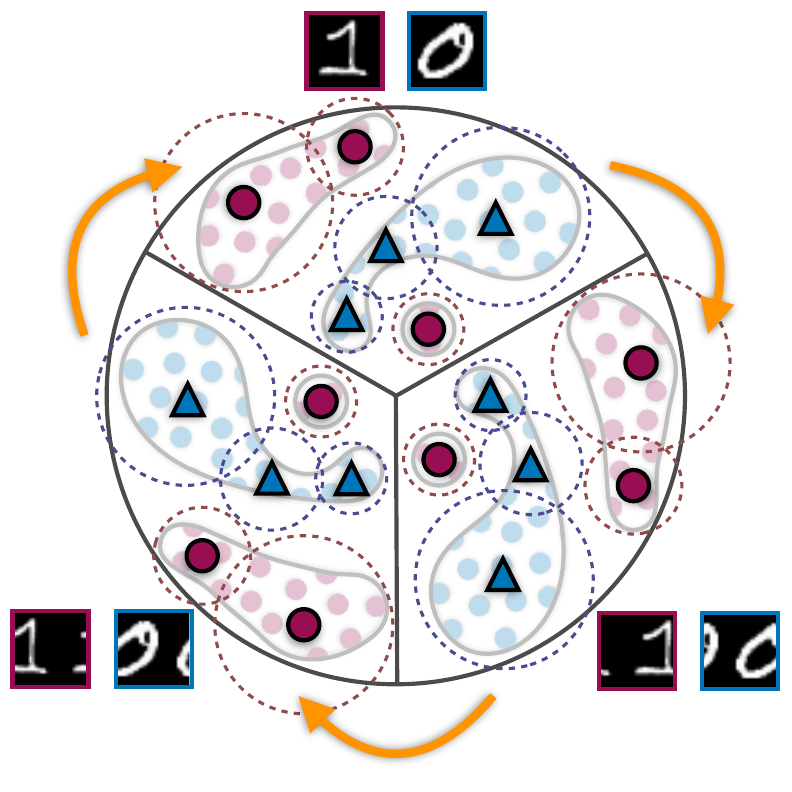}
    \centering
    \caption{\textbf{Translation-invariant mutual covering.} Following an arrow applies a horizontal circular shift. For simplicity, we only include three horizontal translations, in general the division would be more granular and include vertical translations. Condition~\eqref{eq:translation-invariant-radii} on the radii in Theorem~\ref{thm:size-of-translation-invariant-network} lets the components of the covering adapt to the mutual positioning of the classes by allowing them to grow (almost) as far as the optimal decision boundary. At the same time, each component consists of multiple sub-components, one for each possible translation. A component being a union of sub-components makes it translation-invariant.}
    \label{fig:translation-invariant-mutual-covering}
\end{figure}

Employing this pseudo-distance we define the following translation-invariant adaptation of the notion of mutual covering from \citet{dirksen2024}.

\begin{definition}[translation-invariant mutual covering]\label{def:translation-invariant-mutual-covering}
    Recall $\set{X} = \set{X}^+ \uplus \set{X}^-$. We call
    \begin{align*}
        \set{C}^+ = \{\vec{c}_1^+, \ldots, \vec{c}_{M^+}^+\} \subset \set{X}^+, &  \qquad r_1^+, \ldots, r_{M^+}^+ \geq 0,\\
        \set{C}^- = \{\vec{c}_1^-, \ldots, \vec{c}_{M^-}^-\} \subset \set{X}^-, &  \qquad r_1^-, \ldots, r_{M^-}^- \geq 0,
    \end{align*}
    a \emph{translation-invariant mutual covering} of $\set{X}$, if the \emph{components} of the covering,
    \begin{align*}
        \set{X}_k^+ & \coloneqq \{ \vec{x}^+ \in \set{X}^+ : \dist(\vec{c}_k^+, \vec{x}^+) \leq r_k^+ \} = \bigcup_{(h, w) \in \range{H} \times \range{W}}
        \set{X}_{k, (h, w)}^+, \\
        \set{X}_{k, (h, w)}^+ & \coloneqq \set{X}^+ \cap \B_2^d(\T_{h, w} \vec{c}_k^+, r_k^+),
    \end{align*}
    and
    \begin{align*}
        \set{X}_\ell^- & \coloneqq \{ \vec{x}^- \in \set{X}^- : \dist(\vec{c}_\ell^-, \vec{x}^-) \leq r_\ell^- \} = \bigcup_{(h, w) \in \range{H} \times \range{W}} \set{X}_{\ell, (h, w)}^-,\\
        \set{X}_{\ell, (h, w)}^- & \coloneqq \set{X}^- \cap \B_2^d(\T_{h, w} \vec{c}_\ell^-, r_\ell^-),
    \end{align*}
    cover $\set{X}^+$ and $\set{X}^-$, respectively. We call $M^+$ and $M^-$ the \emph{mutual covering numbers}.
\end{definition}

Figure~\ref{fig:translation-invariant-mutual-covering} illustrates the prototypical shape of a translation-invariant mutual covering. The condition~\eqref{eq:translation-invariant-radii} on the radii in Theorem~\ref{thm:size-of-translation-invariant-network} below allows the components to grow as far as the ideal decision boundary between the data classes---up to a logarithmic factor. Consequently, parts of the data far away from the decision boundary can be covered with only a few larger components while smaller components are only required in more complex areas close to the decision boundary. This data-adaptive behavior is similar to the covering of \citet{dirksen2024}. However, in contrast to their approach, our adapted distance measure allows every component to break up into multiple sub-components, each contained inside an Euclidean ball around a translated copy of the components center. In this way, one component not only groups together samples that have close proximity to one another but also those samples that are in close proximity \emph{up to translation}.

We now state the version of our theoretical result for translation-invariant convolutional neural networks, followed by a short proof sketch. The result for more general group invariances is proved later, in Appendix~\ref{sec:group-invariant-memorization-proof}. Below, we denote the \emph{Gaussian width}, for $\set{T} \subset \R^d$, by
\begin{equation*}
    \gw(\set{T}) = \E \sup_{\vec{t} \in \set{T}} \inp{\vec{g}}{\vec{t}},
\end{equation*}
where $\vec{g} \sim \Normal(\vec{0}, \mat{I}_d)$ denotes a standard Gaussian random vector. We refer to \citet[Section~7.5]{vershynin2018high} for a detailed discussion and examples.

\begin{theorem}[size of translation-invariant convolutional net]\label{thm:size-of-translation-invariant-network}
    Let $\set{X} = \set{X}^+ \uplus \set{X}^- \subset \radius \B_2^{L \times H \times W}$ be finite. Suppose there exists a translation-invariant mutual covering with centers $\set{C}^+$ and $\set{C}^-$ satisfying $\dist(\set{C}^+, \set{C}^-) \geq \delta$ and radii satisfying
    \begin{equation}\label{eq:translation-invariant-radii}
        r_k^+ \lesssim \frac{\dist(\vec{c}_k^+, \set{C}^-)}{\sqrt{\log(e \lambda / \dist(\vec{c}_k^+, \set{C}^-))}}
        \qquad\text{and}\qquad
        r_\ell^- \lesssim \frac{\dist(\vec{c}_\ell^-, \set{C}^+)}{\sqrt{\log(e \lambda / \dist(\vec{c}_\ell^-, \set{C}^+))}},
    \end{equation}
    for $k \in \range{M^+}$ and $\ell \in \range{M^-}$. Set $\omega \coloneqq \max\{ \omega^+, \omega^- \}$ where
    \begin{equation*}
        \omega^+ \coloneqq \max_{\substack{k \in \range{M^+} \\ (h, w) \in \range{H} \times \range{W}}} \frac{\gw^2(\set{X}^+_{k, (h, w)})}{\dist^3(\vec{c}_k^+, \set{C}^-)}
        \quad\text{and}\quad
        \omega^- \coloneqq \max_{\substack{\ell \in \range{M^-} \\ (h, w) \in \range{H} \times \range{W}}} \frac{\gw^2(\set{X}^-_{\ell, (h, w)})}{\dist^3(\vec{c}_\ell^-, \set{C}^+)}
    \end{equation*}
    Suppose $\lambda \gtrsim \radius$ and
    \begin{equation*}
        K \gtrsim \lambda \delta^{-1} \log(2 M^+ M^- H W / \eta) + \lambda \omega.
    \end{equation*}
    Then, with probability at least $1 - \eta$, Algorithm~\ref{alg:translation-invariant-memorization} outputs a translation-invariant convolutional neural network $F \colon \R^{L \times H \times W} \to \{ \pm 1 \}$ which interpolates $\set{X}$, that is, for at least a fraction of $\alpha$ samples $\vec{x}^+ \in \set{X}^+$ and all samples $\vec{x}^- \in \set{X}^-$,
    \begin{equation*}
        F(\vec{x}^+) = +1
        \quad\text{and}\quad
        F(\vec{x}^-) = -1.
    \end{equation*}
    Moreover, it uses $K$ filters in its first convolutional layer and no more than $M^+$ filters in its second convolutional layer. In particular, it has no more than $\OO\left( (L + M^+) H W K \right)$ parameters.
\end{theorem}

\begin{algorithm}
\caption{translation-invariant memorization}\label{alg:translation-invariant-memorization}
\begin{algorithmic}[1]
    \Require{$\set{X} = \set{X}^+ \uplus \set{X}^- \subset \R^{L \times H \times W}$ finite with $N^+ \coloneqq \card{\set{X}^+}$ and $N^- \coloneqq \card{\set{X}^-}$, number of random filters $K \in \N$, maximal bias $\lambda \geq 0$, interpolation fraction $\alpha \in (0, 1]$}
    \Ensure{translation-invariant convolutional neural net $F\colon \R^{L \times H \times W} \to \{ \pm 1 \}$}
    \For{$k \in \range{K}$}\label{alg-line:begin-first-layer}
        \State $\vec{g}_k \gets \text{independent sample from } \Normal(\vec{0}, \mat{I}_{L \times H \times W})$
        \State $\mathrlap{\tau_k}\hphantom{\vec{g}_k} \gets \text{independent sample from } \Unif([-\lambda, \lambda])$
    \EndFor\label{alg-line:end-first-layer} \Comment{$\Phi(\vec{x}) = [\sigma(\sum_{\ell \in \range{L}} \vec{x}_\ell \rconv [\vec{g}_k]_\ell + \tau_k)]_{k \in \range{K}}$}\vspace{1ex}
    \State $\set{A}^+ \gets \emptyset$ \label{alg-line:begin-second-layer}
    \State $\mathrlap{\set{C}^+}\hphantom{\set{A}^+} \gets \set{X}^+$
    \While{$\card{\set{C}^+} > 0$}
        \State $\vec{x}_*^+ \gets \text{random sample from } \set{C}^+$ \label{alg-line:select-candidate-sample}
        \State $\mathrlap{\set{C}^+}\hphantom{\vec{x}_*^+} \gets \set{C}^+ \setminus \{ \vec{x}_*^+ \}$
        \State $\mathrlap{\vec{u}_{\vec{x}_*^+}}\hphantom{m_{\vec{x}_*^+}} \gets \ind{\Phi(\vec{x}_*^+) = \vec{0}}$
        \State $m_{\vec{x}_*^+} \gets \min_{\vec{x}^- \in \set{X}^-} \dinp{\Phi(\vec{x}^-)}{\vec{u}_{\vec{x}_*^+}}$ \Comment{$\phihat_{\vec{x}_*^+}(\vec{z}) = \sigma(\sum_{k \in \range{K}} \vec{z}_k \rconv [- \vec{u}_{\vec{x}_*^+}]_k + m_{\vec{x}_*^+})$} \label{alg-line:neuron-definition}
        \State $\mathrlap{\set{V}_{\vec{x}_*^+}}\hphantom{m_{\vec{x}_*^+}} \gets \left\{ \vec{x}^+ \in \set{X}^+ : \GMP \left( \phihat_{\vec{x}_*^+}(\Phi(\vec{x}^+)) \right) > 0 \right\}$
        \If{$\card{\set{V}_{\vec{x}_*^+}} > 0$}
            \State $\set{A}^+ \gets \set{A}^+ \cup \{ \vec{x}_*^+ \}$
            \State $\mathrlap{\set{C}^+}\hphantom{\set{A}^+} \gets \set{C}^+ \setminus \set{V}_{\vec{x}_*^+}$
        \EndIf
    \EndWhile \label{alg-line:end-second-layer} \vspace{1ex}
    \While{$\card{\bigcup_{\vec{x}_*^+ \in \set{A}^+} \set{V}_{\vec{x}_*^+}} \geq \alpha N^+$} \label{alg-line:begin-pruning}
        \State $\mathrlap{\vec{x}_*^+}\hphantom{\set{A}^+} \gets \argmin_{\vec{x}_*^+ \in \set{A}^+} \card{ \set{V}_{\vec{x}_*^+} \setminus \bigcup_{\tilde{\vec{x}}_*^+ \in \set{A}^+ \setminus \{\vec{x}_*^+\}} \set{V}_{\tilde{\vec{x}}_*^+}}$
        \State $\set{A}^+ \gets \set{A}^+ \setminus \{ \vec{x}_*^+ \}$
    \EndWhile
    \State $\set{A}^+ \gets \set{A}^+ \cup \{ \vec{x}_*^+ \}$ \label{alg-line:end-pruning} \Comment{$\hat{\Phi}(\vec{z}) = [\sigma(\sum_{k \in \range{K}} \vec{z}_k \rconv [- \vec{u}_{\vec{x}_*^+}]_k + m_{\vec{x}_*^+})]_{\vec{x}_*^+ \in \set{A}^+}$}\vspace{1ex}
    \State \Return $F(\vec{x}) = \sign(\inp{\vec{1}}{(\lift{\GMP} \circ \hat{\Phi} \circ \Phi)(\vec{x})})$ \Comment{$\vec{1} = (1, \ldots, 1) \in \R^{\set{A}^+}$}
\end{algorithmic}
\end{algorithm}

\begin{proof}[Proof sketch]
    As the predictions of the model outside $\set{X}$ are not important for the interpolation task, we only explicitly encode $\set{X}^+$ with label $+1$ into the network and the complement is assigned label $-1$. This explains the asymmetry between the two classes in the result. In more detail, the algorithm proceeds as follows.
    
    In the first loop of the algorithm (lines~\ref{alg-line:begin-first-layer} to~\ref{alg-line:end-first-layer}), we sample a random convolution layer $\Phi$. Under our assumptions, we may condition on the high probability event of Lemma~\ref{lem:convolution-neuron-activation-region}---the key ingredient of the proof---under which the following holds. After transforming the data set with $\Phi$, for any sample $\vec{x}_*^+ \in \set{X}^+$, the neuron $\phihat_{\vec{x}_*^+}$ defined in line~\ref{alg-line:neuron-definition} (followed by a global max pooling operation) activates on the whole component of the mutual covering of $\set{X}^+$ that this sample resides in. On the other hand, it does not activate on any sample from $\set{X}^-$. Crucially, this holds for any sample within any of the components, which allows us to simply select candidate samples at random (line~\ref{alg-line:select-candidate-sample}). Hence, in the second loop (lines~\ref{alg-line:begin-second-layer} to~\ref{alg-line:end-second-layer}) we select a random candidate sample $\vec{x}_*^+ \in \set{C}^+$ and, after computing the activation region $\set{V}_{\vec{x}_*^+}$ of the neuron $\phihat_{\vec{x}_*^+}$, we remove the entire activation region from $\set{C}^+$. This eliminates at least one component of this mutual covering. Thus, in the next step, a candidate sample outside the just eliminated component is picked at random, eliminating another component by the same principle. By induction, the loop terminates after at most $M^+$ steps, producing a second convolutional layer with at most $M^+$ filters. At this point the network already fully interpolates $\set{X}$. To see this, notice that the activation regions $\set{V}_{\vec{x}_*^+}$ of the second layer neurons (together with the global max pooling operation) cover $\set{X}^+$. The last linear layer simply sums over the output of the global max pooling layer and outputs $+1$ if the input of the network is in any of the activation regions, and $-1$ otherwise. At this point we can already count the number of parameters.
    
    The third loop (lines~\ref{alg-line:begin-pruning} to~\ref{alg-line:end-pruning}) uses a heuristic rule to successively prune the neurons in the second layer which contribute the least to the interpolation of $\set{X}^+$. It stops exactly when pruning the next neuron would lead to interpolating less than a fraction of $\alpha$ samples. Notice that $\alpha$ is not reflected in the bound on the size of the network. Nevertheless, we observe its beneficial effects in the experiments, where it reduces the size of the network.
    
    The translation invariance is built into the architecture. The convolutional layers are translation-equivariant. Hence, when inputting a translated sample into the network, these layers simply pass on the translation. The global max pooling layer, on the other hand, is translation-invariant. Thus, any translation that was passed through the network until this point has no effect after this layer.
\end{proof}

\section{Experimental results}
\label{sec:experiments}

In this section, we present a series of experiments\footnote{The code is available at \url{https://github.com/patrickfinke/invariant-memo}. We ran the simulations on a RTX 5080 with 16GB of VRAM.} in which we evaluate the performance of our algorithm and compare it to the previous work of \citet{dirksen2024}. In the previous section, we presented our translation-invariant construction. To model translation-invariance with a stride greater than one or other groups of transformations altogether we use our generalized Algorithm~\ref{alg:group-invariant-memorization-multiclass} which we introduce in full detail in Appendix~\ref{sec:group-invariant-memorization}. For multi-class problems, we employ an extension which is similar to common one-vs-rest methods, see Appendix~\ref{sec:multiclass-algorithm}. To reproduce the fully-connected construction of \citet{dirksen2024}, we employ the generalized algorithm and simply choose the trivial group. In the following, the `fully-connected network' refers to this construction and the `convolutional network' refers to our invariant construction.

In Section~\ref{sec:experiments/2d}, we begin with a two-dimensional example. This provides some visual intuition on the structure of the networks that our method produces. In Section~\ref{sec:experiments/mnist-and-cifar}, we evaluate the performance of our algorithm on both MNIST \citep{lecun1998mnist} and CIFAR-10 \citep{krizhevsky2009learning}. In Section~\ref{sec:experiments/pruning}, this is followed up by examining the influence of the interpolation fraction hyperparameter $\alpha$. This yields some understanding of how well-conditioned MNIST and CIFAR-10 are with respect to our method. Up until this point, any invariance in the data originated from the distributions. Section~\ref{sec:experiments/invariant-data} features a more controlled setting, where we augment the data sets to be perfectly invariant by including all possible transformations of the samples.

\subsection{Illustrative two-dimensional example}
\label{sec:experiments/2d}

In order to gain some geometric intuition, let us briefly step away from the high-dimensional image setting in favor of a two-dimensional toy example. In such a setting, we can easily visualize the activation regions of the individual neurons just before the last linear layer. This offers a detailed look at the decision boundary of the network, as the union of these regions is exactly the region where the network outputs a positive label.

\begin{figure}[!tb]
    \null\hfill%
    \begin{subfigure}[b]{0.45\textwidth}
        \centering
        \includegraphics{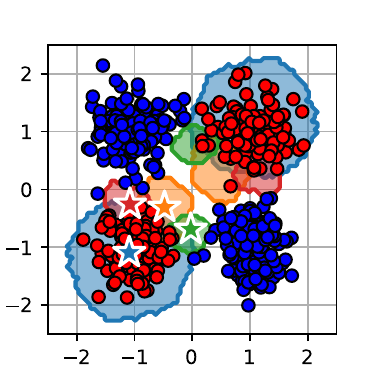}
        \caption{$\set{G} = \{ \mat{I}_2, -\mat{I}_2 \}$}
        \label{fig:experiments/2d/invariant}
    \end{subfigure}
    \hfill
    \begin{subfigure}[b]{0.45\textwidth}
        \centering
        \includegraphics{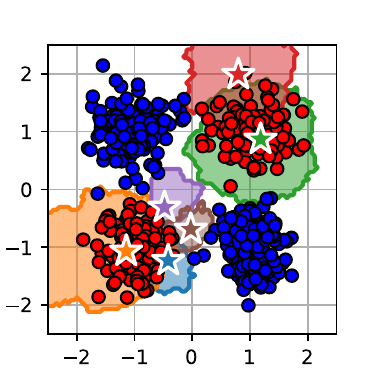}
        \caption{$\set{G} = \{ \mat{I}_2 \}$}
        \label{fig:experiments/2d/fully-connected}
    \end{subfigure}
    \null\hfill%
    \caption{\textbf{Activation regions on the XOR data set.} Red and blue points have labels $+1$ and $-1$, respectively. For each class we have drawn $250$ samples from a mixture of two Gaussian distributions with covariance matrices $0.35 \cdot \mat{I}_2$. The associated reference samples which the algorithm picked at random are marked with stars of the same color as its neuron's activation region. On the left, the network is invariant under $180\degree$ rotations, while on the right, there is no invariance.}
    \label{fig:experiments/2d}
\end{figure}

Figure~\ref{fig:experiments/2d} shows an instance of the XOR data set. Evidently, a classifier for this problem can benefit from a built-in invariance with respect to $180\degree$ rotations around the origin. In Appendix~\ref{sec:group-invariant-memorization}, we model a general invariance with respect to a finite group of orthogonal transformations. In Figure~\ref{fig:experiments/2d/invariant}, we choose $\set{G} = \{ \mat{I}_2, - \mat{I}_2 \}$ which represents the aforementioned rotational invariance. We see that the algorithm produces neurons with activation regions consisting of two parts which, together, reflect that invariance. We can compare this behavior to the fully-connected construction of \citet{dirksen2024}, which we recover via the trivial group $\set{G} = \{ \mat{I}_2 \}$, in Figure~\ref{fig:experiments/2d/fully-connected}. As this does not model the invariance of the problem, their method has to cover both quadrants of the positive class individually.

\subsection{Performance on real-world image data}
\label{sec:experiments/mnist-and-cifar}

In this section, we evaluate and compare the performance of our algorithm on the MNIST and CIFAR-10 data sets. Inspired by the experiments of \citet{dirksen2024}, we pick the maximal bias $\lambda = 0$ as this choice provided the best results in their experiments, among $\lambda \in [0, 3/2]$.\footnote{\citet{dirksen2024} normalized the $\ell_2$-radius of both MNIST and CIFAR-10 to $\radius = 1$, which we do here as well.} For $K$, we first choose values $K_{\text{ConvNet}}$ for the convolutional network in a large enough range such that the number of filters in the second layer converges. In a convolutional layer, depending on the data dimension and the stride of the convolution, each filter is shifted to $\hat{d}$ different positions. To produce a comparable result in which the first layers of both networks have the same output dimension, for the fully-connected network we choose $K_{\text{FC}} = K_{\text{ConvNet}} \cdot \hat{d}$.

A stride larger than one for the random convolution layer is modeled using our generalized theory from Appendix~\ref{sec:group-invariant-memorization}. This slightly weakens the translation-invariant property of the network. For example, for a stride of two, the network is only invariant to translations of two pixels at a time. However, there are benefits to this, \eg, the required number of parameters is lower, as we discuss below.

\paragraph{MNIST.}

\begin{figure}[!tb]
    \centering
    \includegraphics{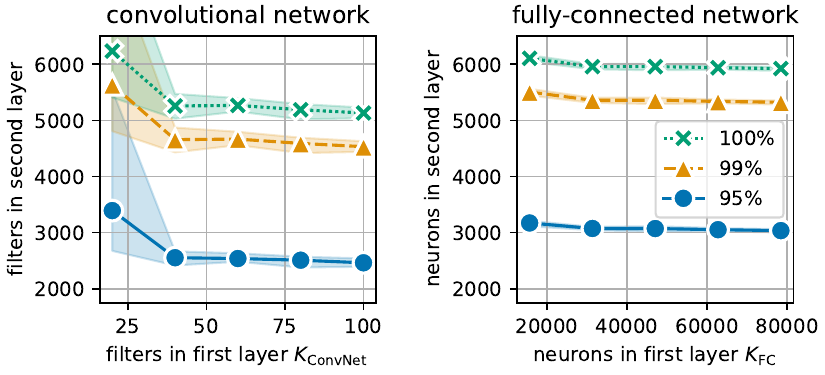}
    \caption{\textbf{MNIST: size of the second layer.} Depending on the number of filters (or neurons) in the first layer we report the number of filters (or neurons) in the second layer for three different levels of $\alpha$. We report the average over five runs with $95\%$ percentile intervals.}
    \label{fig:experiments/mnist/components}
\end{figure}

In this experiment, we pick a stride of one for the random convolution layer. On the left of Figure~\ref{fig:experiments/mnist/components}, we see that for the convolutional network increasing the number of filters $K_{\text{ConvNet}}$ in the random convolutional layer decreases the number of filters for the second layer, until there are diminishing returns at around $K_{\text{ConvNet}} = 60$. The fully-connected network behaves similarly, as shown on the right side of the same figure, where increasing the number of neurons $K_{\text{FC}}$ in the random fully-connected layer decreases the number of neurons in the second layer with diminishing returns somewhere around $K_{\text{FC}} = 30\,000$. Recall that for large $K$ the number of filters and neurons in the second layer, respectively, roughly corresponds to the mutual covering number. Therefore, the results indicate that the translation-invariant mutual covering has less components than its non-invariant counterpart, which means it captures the structure of the data set more concisely. Moreover, the convolutional network converges more efficiently, as both networks (due to weight sharing in the convolutional net) require $K \cdot d$ parameters in their first layers, where $d = 28 \cdot 28$ is the dimension of MNIST.

\begin{figure}[!tb]
    \centering
    \includegraphics{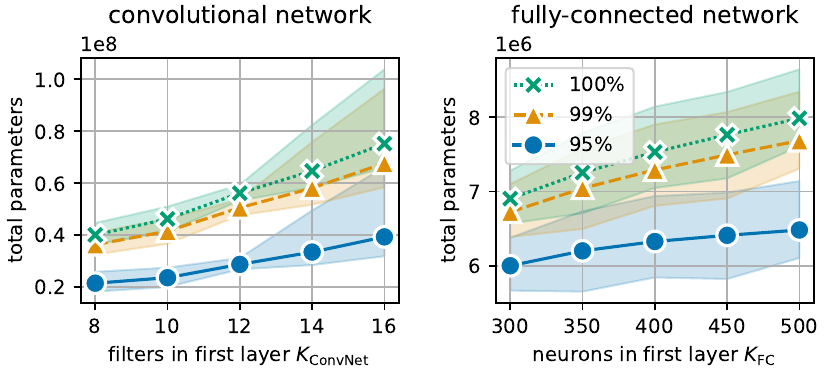}
    \caption{\textbf{MNIST: number of parameters.} Depending on the number of filters (or neurons) in the first layer we report the total number of parameters of each network for three different levels of $\alpha$. We report the average over ten runs with $95\%$ percentile intervals. In both cases, decreasing the size of the first layer by one more step leads to a failed interpolation attempt in at least one of the ten random runs.}
    \label{fig:experiments/mnist/parameters}
\end{figure}

In Figure~\ref{fig:experiments/mnist/parameters}, to reduce the number of total parameters, we see that for both networks it is beneficial to choose $K$ as small as possible (under the condition that the method still interpolates). From the algorithms we can deduce that the fully-connected network requires $K_{\text{FC}} + 1$ parameters per neuron (including the bias) in the second layer while the convolutional network requires $K_{\text{ConvNet}} \cdot d + 1$ parameters per filter in the second layer. Thus, in order for the convolutional network to require a lower total number of parameters, the lower possible value of $K_{\text{ConvNet}}$ and the corresponding lower number of filters in the second layer would need to undercut the influence of the factor $d$. However, as we can infer from the figures, this is not the case.

Finally, in all figures we can observe the influence of the interpolation fraction $\alpha$. We see that reducing $\alpha$ to $95\%$ cuts the sizes of the second layers roughly in half. We discuss this further in Section~\ref{sec:experiments/pruning}.

\paragraph{CIFAR-10.}

\begin{figure}[!tb]
    \centering
    \includegraphics{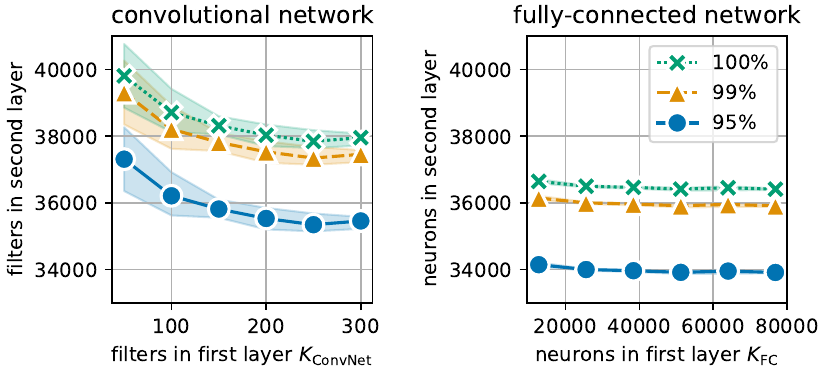}
    \caption{\textbf{CIFAR-10: size of the second layer.} Depending on the number of filters (or neurons) in the first layer we report the number of filters (or neurons) in the second layer for three different levels of $\alpha$. We report the average over five runs with $95\%$ percentile intervals.}
    \label{fig:experiments/cifar/components}
\end{figure}

In this experiment, we pick a stride of two for the random convolution layer. In Figure~\ref{fig:experiments/cifar/components} we again observe that an increase in $K$ leads to a decrease in the number of filters and neurons in the second layer, respectively. There are diminishing returns in both cases, around $K_{\text{ConvNet}} = 250$ and $K_{\text{FC}} = 40\,000$. As in the case of MNIST, both networks require $K \cdot d$ parameters in the first layer, where $d = 3 \cdot 32 \cdot 32$ is the dimension of CIFAR-10. Again, the second layer of the convolutional network converges more efficiently in this sense. However, the convolutional network requires slightly more second layer filters than the fully-connected network requires neurons in its second layer. Recall again that these roughly correspond to the mutual covering numbers. As we can see in Definition~\ref{def:translation-invariant-mutual-covering}, while the components for the convolutional network consist of a union of multiple sub-components, their individual radii can become smaller compared to the fully-connected counterpart. This interplay might have a non-trivial effect on the covering number and we suspect that exactly this can be observed here.

\begin{figure}[!tb]
    \centering
    \includegraphics{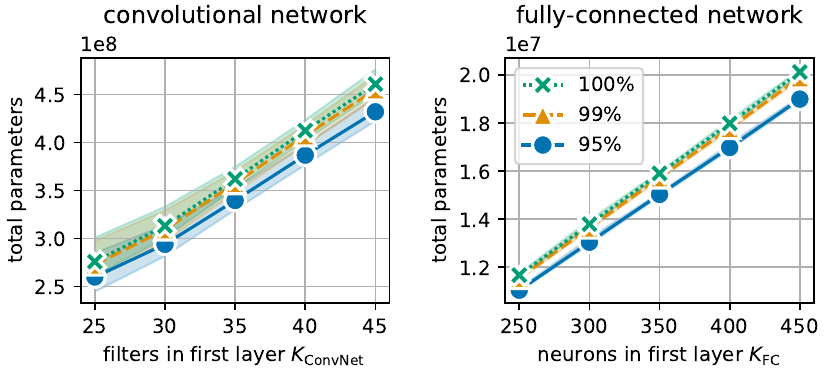}
    \caption{\textbf{CIFAR-10: number of parameters.} Depending on the number of filters (or neurons) in the first layer we report the number of total parameters of each network for three different levels of $\alpha$. We report the average over ten runs with $95\%$ percentile intervals. In both cases, decreasing the size of the first layer by one more step leads to a failed interpolation attempt in at least one of the ten random runs.}
    \label{fig:experiments/cifar/parameters}
\end{figure}

Due to the stride of two we only have $\hat{d} = 16 \cdot 16$. However, as in the case of MNIST, we still observe more parameters for the convolutional network in Figure~\ref{fig:experiments/cifar/parameters}.

Again, we can see that reducing $\alpha$ reduces the sizes of the second layers, but the effect is not as pronounced as for MNIST. We analyze this further in Section~\ref{sec:experiments/pruning} below.

\subsection{Estimating how well a data set is conditioned}
\label{sec:experiments/pruning}

The idea behind the mutual covering is that it can adapt to the structure of a given problem by covering many samples far away from the decision boundary with only a few large components, while requiring smaller components only in more challenging areas to cover the remaining samples. One could say that such a mutual covering is `good' while, \eg, a trivial covering where every sample is covered by its own dedicated component is certainly `bad'. In the same way, we might think of a data set as well-conditioned for our interpolation algorithm if it has a `good' covering and as ill-condition if it only has `bad' coverings.

Recall once again that, for large values of $K$, the number of filters in the second layer of the convolutional network and the number of neurons in the second layer of the fully-connected network roughly correspond to the respective covering number. Hence, observing how many filters or neurons are needed by our algorithm to interpolate a fraction $\alpha$ of samples can give an empirical hint of the size of the components and thus on how well-conditioned a data set is. In the most ill-conditioned case, our method will require one filter (respectively neuron) per sample, corresponding to a linear increase in terms of $\alpha$.

\begin{figure}[!tb]
    \null\hfill%
    \begin{subfigure}[b]{0.5\textwidth}
        \centering
        \includegraphics{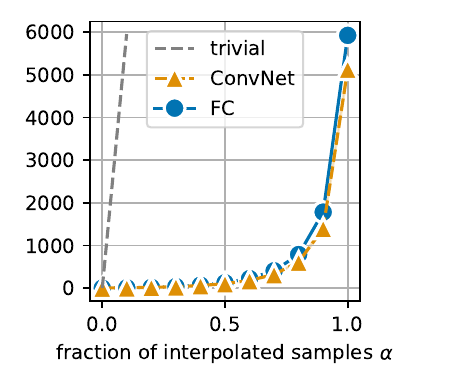}
        \caption{MNIST}
        \label{fig:experiments/pruning/mnist}
    \end{subfigure}%
    \begin{subfigure}[b]{0.5\textwidth}
        \centering
        \includegraphics{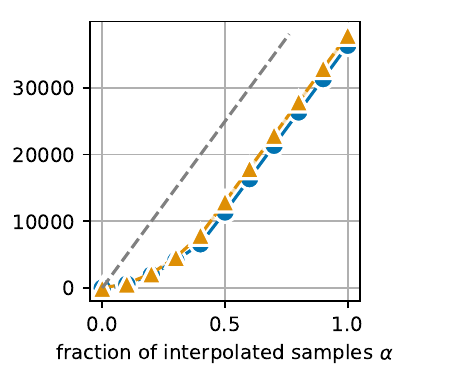}
        \caption{CIFAR-10}
        \label{fig:experiments/pruning/cifar}
    \end{subfigure}
    \null\hfill%
    \caption{\textbf{Influence of the interpolation fraction hyperparameter.} We report the number of filters and neurons in the second layer, respectively, in terms of the interpolation fraction $\alpha$. In each case, we picked the largest values of $K$ from Figures~\ref{fig:experiments/mnist/components} and~\ref{fig:experiments/cifar/components}. We report the average over five runs with $95\%$ percentile intervals, which are very narrow and thus hard to see.}
    \label{fig:experiments/pruning}
\end{figure}

Let us observe beforehand, in Figure~\ref{fig:experiments/pruning}, that both the convolutional network and the fully-connected network behave very similarly regarding the interpolation fraction $\alpha$. A likely explanation is that both MNIST and CIFAR-10 have been preprocessed such that its subjects are centered, hence there are no samples that are translated away from the origin. In this case, we cannot expect the convolutional network to be more efficient at interpolating these samples than the fully-connected network. Note, however, that the convolutional network is still more expressive as it extrapolates to any unseen translates of samples in the data set via the built-in translation-invariance.

In Figure~\ref{fig:experiments/pruning/mnist}, for MNIST we see a gradual increase in the required number of filters up until around $\alpha = 90\%$ at which point the behavior becomes trivial, where the incline matches that of a network which represents only one sample per neuron. On the other hand, in Figure~\ref{fig:experiments/pruning/cifar}, for CIFAR-10 the trivial behavior already starts at around $\alpha = 40\%$, hinting that CIFAR-10 is less well-conditioned than MNIST and harder to interpolate for our method.

\subsection{Invariant data sets}
\label{sec:experiments/invariant-data}

\begin{table}[]
    \centering
    \begin{tabular}{c|c|c|c}
        data set & $\set{G}$ & $\card{\set{G}}$ & $\overline{M}_{\text{FC}} / \overline{M}_{\text{ConvNet}}$ \\ \hline
        XOR & $\{ \mat{I}_2, -\mat{I}_2 \}$ & $2$ & $2.07$ \\
        MNIST & $\{ \lift{\T}_{0, 0}, \lift{\T}_{0, 7}, \lift{\T}_{0, 14}, \lift{\T}_{0, 21} \}$ & $4$ & $3.98$ \\
        CIFAR-10 & $\{ \lift{\T}_{0, 0}, \lift{\T}_{0, 8}, \lift{\T}_{0, 16}, \lift{\T}_{0, 24} \}$ & $4$ & $3.92$ \\ \cdashline{1-4}
        XOR$^*$ & $\{ \mat{I}_2, -\mat{I}_2 \}$ & $2$ & $1.61$ \\
    \end{tabular}\\
    \caption{\textbf{Difference in covering numbers on invariant data sets.} Ratio of the average number of neurons in the second layer of the fully-connected network $\overline{M}_{\text{FC}}$ and average number of filters in the second layer of the convolutional network $\overline{M}_{\text{ConvNet}}$ compared to the group size. In any case we use $K_{\text{ConvNet}} = 5\,000$ and $K_{\text{FC}} = \card{\set{G}} \cdot K_{\text{ConvNet}}$, similar to Section~\ref{sec:experiments/mnist-and-cifar}. For XOR and XOR$^*$ we used $\lambda = 2$ and did $50$ random trials, for MNIST and CIFAR-10 we used $\lambda = 0$ and did $5$ random trials. XOR$^*$ uses only the plain data set.}
    \label{tab:reduction}
\end{table}

In the previous subsections we observed that that the non-invariant covering number and our translation-invariant covering number (approximated by the number of neurons and filters in the second layer of the respective network) behave very similarly. We conjectured that this might be due to missing translation-invariance in the data sets due to preprocessing of the images.

What happens if there is enough invariance in the data set? Let $\set{G}$ be a finite group of orthogonal transformations, such as image-space translations in the case of MNIST and CIFAR-10 or rotations in the case of XOR. Each of the components of the $\set{G}$-invariant covering possess $\card{\set{G}}$ sub-components. If the data is invariant, these cover transformed copies of the same samples, all of which the non-invariant covering needs to cover, too. Hence, one might hope that the invariant covering number is by a factor $\card{\set{G}}$ smaller than its non-invariant counterpart. However, due to the adaptive nature of the component radii, this is not obvious and could be dataset-dependent.

To test this intuition, for a given data set $\set{X}$, we compare both methods on
\begin{equation*}
    \set{X}_{\set{G}}
    \coloneqq \{ Q \vec{x} : Q \in \set{G}, \vec{x} \in \set{X} \},
\end{equation*}
which contains all possible transformations of the original samples. Note that for our method with built-in invariance simply using $\set{X}$ is equivalent.

The results in \cref{tab:reduction} show that our initial intuition seems largely correct. The ratio of the average number of neurons in the second layer of the fully-connected network $\overline{M}_{\text{FC}}$ and the average number of filters in the second layer of the convolutional network $\overline{M}_{\text{ConvNet}}$ tracks $\card{\set{G}}$ closely over all tested data sets. Additionally, in the last row, we consider the plain XOR data set. Here, the invariance is already naturally encoded in the data distribution. We execute both methods on the unaugmented data set $\set{X}$ while still encoding $\set{G}$ into the convolutional model. We notice a clear improvement over the non-invariant covering, but it is not as large as when considering the `perfectly invariant' data set $\set{X}_\set{G}$. This makes sense, as the invariance is in the distribution only and hence `noisy'.

\section{Conclusion}
\label{sec:conclusion}

In this paper, we presented an instance-specific analysis of the memorization capacity of translation-invariant convolutional neural networks with threshold activations. For a data set of two classes, we derived bounds on the required number of neurons and parameters for such a network to be able to interpolate the data. These bounds are linked to a new notion of translation-invariant mutual covering which respects the translation-invariant property of the network and takes both the mutual arrangement of the two classes and their geometric properties into account.

We derived these results as a special case of a general framework, where we considered invariance with respect to an arbitrary finite group of orthogonal transformations. This includes multi-dimensional convolutions, reflections, and rotations. We gave a constructive proof by introducing a randomized algorithm which for given input data computes instances of such networks, with high probability. We conducted numerical experiments where we investigate the empirical performance of our algorithm and compare it to previous work. While the produced networks were not smaller than previous non-translation invariant networks they are more expressive due to their built-in invariance.

\section*{Use of AI tools}

The authors used AI solely to improve the wording and grammar of self-written prose. AI was not used for any other purpose, including research, proofs and other mathematics, figures, or code. A preliminary version of the results was announced in the PhD thesis of the second-named author \citep{finkeThesis}.

\section*{Acknowledgments and Disclosure of Funding}
PG is grateful for the hospitality of Utrecht University, where he worked on this project as a guest researcher. PG acknowledges support by the German Science Foundation (DFG) in the context of the Emmy Noether junior research group VO 2594/1--1.

\newpage

\appendix

\section{Generalized theory}
\label{sec:group-invariant-memorization}

In this section, we introduce our theoretical result in its full generality. To this end, we use a form of group equivariant convolution which was first introduced by \citet{cohen2016group}, although with a different notation. We reprove all the properties of this generalized convolution that are needed for our analysis.

In what follows, $\set{G} \leq \set{O}(d)$ denotes a finite subgroup of the orthogonal group. We aim to produce a model $F\colon \R^d \to \{\pm 1\}$ that is invariant under the transformations contained in $\set{G}$, that is, for all $\vec{x} \in \R^d$,
\begin{equation*}
    F(Q \vec{x}) = F(\vec{x})
    \quad \text{for all $Q \in \set{G}$}.
\end{equation*}

In the main part, we consider the special case of a translation-invariant convolutional neural network which is applied to 2D images with $L$ color channels, each of size $H \times W$. This corresponds to the choice
\begin{equation*}
    \set{G}
    = \set{G}_{\T}
    \coloneqq \{ \lift{\T}_{h, w} : (h, w) \in \range{H} \times \range{W} \},
\end{equation*}
where, as in the main part, $\T_{h, w}$ denotes the circular translation in pixel-space by $h$ pixels vertically and $w$ pixel horizontally, and the translation is applied for each color channel via the lifting $\lift{\T}_{h, w}$, defined in \eqref{eq:lifted-translation}. Clearly, these transformations are orthogonal as they are permutations, that is, a reordering of the elements of a vector. The same holds true for horizontal and vertical flips, $90\degree$ rotations for square 2D images or translations in higher dimensions. Over the course of this section, we highlight the special case of $\set{G}_{\T}$. The result from the main part follows from this discussion.

\subsection{Notation}
\label{sec:notation-extended}

We introduce an extended vector notation, which is similar to the notation used for tensors or multidimensional arrays. In the following, let $\set{I}$, $\set{J}$, and $\set{K}$ denote finite index sets.

We denote by $\R^\set{I}$ the space of vectors with entries in $\R$ that are indexed by the elements of $\set{I}$. Evidently, this space is isomorphic to $\R^{\card{\set{I}}}$. Let $\vec{x}, \vec{y} \in \R^\set{I}$. We denote by $x_i$ the $i$-th entry of the vector $\vec{x}$. We denote the $\ell_p$-norm, for $1\leq p<\infty$, and the Euclidean inner product by
\begin{equation*}
    \norm{\vec{x}}_p = \left( \sum_{i \in \set{I}} |x_i|^p \right)^{1/p}
    \quad\text{and}\quad
    \inp{\vec{x}}{\vec{y}} = \sum_{i \in \set{I}} x_i y_i,
\end{equation*}
respectively. In case of a double index, for $\vec{v} \in \R^{\set{I} \times \set{J}}$, we select subsets of entries via
\begin{equation*}
    \vec{v}_i \coloneqq \vec{v}_{(i, :)} \coloneqq [v_{(i, j)}]_{j \in \set{J}} \in \R^\set{J}
    \quad\text{and}\quad
    \vec{v}_{(:, j)} \coloneqq [v_{(i, j)}]_{i \in \set{I}} \in \R^\set{I},
\end{equation*}
where we fixed the first and second index, respectively. For $f \colon \R^\set{J} \to \R^\set{K}$ we define the \emph{lifting}
\begin{equation*}
    \lift{f} \colon \R^{\set{I} \times \set{J}} \to \R^{\set{I} \times \set{K}},
    \quad \vec{v} \mapsto [f(\vec{v}_i)]_{i \in \set{I}}
\end{equation*}
as the repeated application of $f$ over the first index. The set of indices $\set{I}$ over which we lift will in each case be clear from the context. Moreover, we slightly abuse the above notation by identifying $n \in \N$ with $\range{n} = \{1, \ldots, n\}$. Hence, we might write $\R^{n \times \set{I}} \coloneqq \R^{\range{n} \times \set{I}}$ for the space of vectors indexed by the elements of $\range{n} \times \set{I}$.

A double index will be helpful when working with convolution layers where we want to group entries that have been manipulated by the same filter. In general we do not need more than two indices, for example, any color channel of a 2D image with $H \times W$ pixels can be simply identified with a vector of dimension $d = H \cdot W$.

Let us conclude with a short example of the above notation. Let $K \in \N$ and $f \colon \R^\set{G} \to \R^\set{G}$. For $\vec{v}, \vec{u} \in \R^{K \times \set{G}}$, we may write
\begin{equation*}
    \inp{\vec{v}}{\lift{f}(\vec{u})}
    = \sum_{k \in \range{K}} \inp{\vec{v}_k}{[\lift{f}(\vec{u})]_k}
    = \sum_{k \in \range{K}} \inp{\vec{v}_k}{f(\vec{u}_k)},
\end{equation*}
where we first make the summation over $k \in \range{K}$ explicit and then resolve the lifting.

\subsection{Generalized memorization algorithm}
\label{sec:group-invariant-memorization-algorithm}

In this section, we introduce Algorithm~\ref{alg:group-invariant-memorization}, which---with high probability---produces a model that interpolates a given data set while being invariant under the transformations contained in $\set{G}$. We make some necessary definitions to state Algorithm~\ref{alg:group-invariant-memorization} and, along the way, we highlight how Algorithm~\ref{alg:group-invariant-memorization} reduces to Algorithm~\ref{alg:translation-invariant-memorization} in the special case of $\set{G}_{\T}$.

First, we generalize the notion of circular convolution.
\begin{definition}
    Let $\vec{x}, \vec{g} \in \R^d$. We call $\vec{x} \conv \vec{g} \in \R^\set{G}$ the \emph{circular convolution} with respect to $\set{G}$ and define it element-wise by
    \begin{equation*}
        [\vec{x} \conv \vec{g}]_Q
        = \inp{\vec{x}}{Q \vec{g}},
        \qquad Q \in \set{G}.
    \end{equation*}
\end{definition}

In the case of $\set{G}_{\T}$ the above corresponds to the regular 2D multi-channel circular convolution.\footnote{Strictly speaking, this results in the cross-correlation, which, in contrast to the convolution, is not commutative. To obtain the convolution one would need to additionally `flip' the kernel. In the literature, often both are referred to as convolution---as done by, \eg, \citet{goodfellow2016}---and moreover many deep learning software libraries only implement the cross-correlation. As the distribution of the random weights in our construction does not change when they are `flipped' both terminologies are equivalent in the first layer. For the second layer, one would indeed need to `flip' the weights.} To see this, we identify $(h, w) \in \range{H} \times \range{W}$ with $\lift{\T}_{h, w} \in \set{G}_{\T}$, so that $\R^{H \times W} \simeq \R^{\set{G}_{\T}}$. Then, for $\vec{x}, \vec{g} \in \R^{L \times H \times W}$ and each entry $(h, w) \in \range{H} \times \range{W}$,
\begin{equation*}
    [\vec{x} \conv \vec{g}]_{\lift{\T}_{h, w}}
    = \inp{\vec{x}}{\lift{\T}_{h, w} \vec{g}}\\
    \doteq  \sum_{\ell \in \range{L}} \inp{\vec{x}_\ell}{\T_{h, w} \vec{g}_\ell}\\
    = \left[ \sum_{\ell \in \range{L}} \vec{x}_\ell \rconv \vec{g}_\ell \right]_{(h, w)}.
\end{equation*}
Hence, we have $[\vec{x} \conv \vec{g}] \doteq \sum_{\ell \in \range{L}} \vec{x}_\ell \rconv \vec{g}_\ell$, which shows that the computation of the first layer in Algorithms~\ref{alg:translation-invariant-memorization} and~\ref{alg:group-invariant-memorization} is identical.

\begin{definition}\label{def:circular-shift-operator}
    Let $R \in \set{G}$ and $\vec{v} \in \R^\set{G}$. We define the \emph{circular shift operator} $\shift_R\colon \R^\set{G} \to \R^\set{G}$ induced by $R$ by
    \begin{equation*}
        [\shift_R \vec{v}]_Q
        = \vec{v}_{R^\top Q},
        \quad Q \in \set{G}.
    \end{equation*}
\end{definition}
Throughout we will use the elementary properties
\begin{align}\label{eqn:elemShiftProp}
    \shift_R\shift_U = \shift_{RU},
    \qquad \shift_R^\top = S_{R^\top},
    \qquad \|\shift_R \vec{v}\|_p = \|\vec{v}\|_p \qquad \text{for all } 1\leq p\leq\infty,    
\end{align}
for all $R,U\in \set{G}$ and $\vec{v}\in \R^{\set{G}}$. These properties clearly transfer to the lifting $\lift{\shift}_R$.  

The last piece of notation needed to state Algorithm~\ref{alg:group-invariant-memorization} is the following. For $\vec{v}, \vec{u} \in \R^{K \times \set{G}}$, we define the kernel
\begin{equation*}
    \dinp{\vec{v}}{\vec{u}}
    \coloneqq \min_{R \in \set{G}} \inp{\vec{v}}{\lift{\shift}_R \vec{u}}.
\end{equation*}

Let us conclude by showing how, in the case of $\set{G}_{\T}$, the computation of the second layer in Algorithm~\ref{alg:translation-invariant-memorization} and~\ref{alg:group-invariant-memorization} is identical. With the same identification as before, for each entry  $(h', w') \in \range{H} \times \range{W}$ and $\vec{v} \in \R^\set{G}$,
\begin{equation*}
    [\shift_{\lift{\T}_{h, w}} \vec{v}]_{\lift{\T}_{h', w'}}
    = \vec{v}_{\lift{\T}_{-h, -w} \lift{\T}_{h', w'}}
    = \vec{v}_{\lift{\T}_{h' - h, w' - w}}
    \doteq \vec{v}_{(h' - h, w' - w)}
    = [\T_{h, w} \vec{v}]_{(h', w')}.
\end{equation*}
Hence, we have $\shift_{\lift{\T}_{h, w}} \doteq \T_{h, w}$. Moreover, for $\vec{z}, \vec{u} \in \R^{K \times \set{G}_{\T}}$,
\begin{equation*}
    \inp{\vec{z}}{\lift{\shift}_{\lift{\T}_{h, w}} \vec{u}}
    = \sum_{k \in \range{K}} \inp{\vec{z}_k}{\shift_{\lift{\T}_{h, w}} \vec{u}_k}
    \doteq \sum_{k \in \range{K}} \inp{\vec{z}_k}{\T_{h, w} \vec{u}_k}
    = \left[ \sum_{k \in \range{K}} \vec{z}_k \rconv \vec{u}_k \right]_{(h, w)}
\end{equation*}
which further yields
\begin{equation*}
     - \dinp{\vec{z}}{\vec{u}}
     = \max_{\lift{\T}_{h, w} \in \set{G}_{\T}} \inp{\vec{z}}{\lift{\shift}_{\lift{\T}_{h, w}} (-\vec{u})}\\
     \doteq \max_{(h, w) \in \range{H} \times \range{W}} \left[ \sum_{k \in \range{K}} \vec{z}_k \rconv [- \vec{u}]_k \right]_{(h, w)}.
\end{equation*}
Hence, the neurons computed in Algorithm~\ref{alg:group-invariant-memorization} are
\begin{equation*}
    \phihat_{\vec{x}_*^+}(\vec{z})
    \doteq \GMP \left\{ \sigma \left( \sum_{k \in \range{K}} \vec{z}_k \rconv [- \vec{u}]_k + m_{\vec{x}_*^+} \right) \right\}.
\end{equation*}
Finally, notice that in the formulation of Algorithm~\ref{alg:translation-invariant-memorization} we detached $\GMP$ from the second layer in favor of a separate global max pooling layer.

\begin{algorithm}
\caption{$\set{G}$-invariant memorization}\label{alg:group-invariant-memorization}
\begin{algorithmic}[1]
    \Require{$\set{G} \leq \set{O}(d)$ finite, $\set{X} = \set{X}^+ \uplus \set{X}^- \subset \R^d$ finite with $N^+ \coloneqq \card{\set{X}^+}$ and $N^- \coloneqq \card{\set{X}^-}$, number of random filters $K \in \N$, maximal bias $\lambda \geq 0$, interpolation fraction $\alpha \in (0, 1]$}
    \Ensure{$\set{G}$-invariant (generalized) convolutional neural net $F\colon \R^d \to \{\pm 1\}$}
    \For{$k \in \range{K}$}
        \State $\vec{g}_k \gets \text{independent sample from } \Normal(\vec{0}, \mat{I}_d)$
        \State $\mathrlap{\tau_k}\hphantom{\vec{g}_k} \gets \text{independent sample from } \Unif([-\lambda, \lambda])$
    \EndFor \Comment{$\Phi(\vec{x}) = [\sigma(\vec{x} \conv \vec{g}_k + \tau_k)]_{k \in \range{K}}$}\vspace{1ex}
    \State $\set{A}^+ \gets \emptyset$
    \State $\mathrlap{\set{C}^+}\hphantom{\set{A}^+} \gets \set{X}^+$
    \While{$\card{\set{C}^+} > 0$}
        \State $\vec{x}_*^+ \gets \text{random sample from } \set{C}^+$
        \State $\mathrlap{\set{C}^+}\hphantom{\vec{x}_*^+} \gets \set{C}^+ \setminus \{ \vec{x}_*^+ \}$
        \State $\mathrlap{\vec{u}_{\vec{x}_*^+}}\hphantom{m_{\vec{x}_*^+}} \gets \ind{\Phi(\vec{x}_*^+) = \vec{0}}$
        \State $m_{\vec{x}_*^+} \gets \min_{\vec{x}^- \in \set{X}^-} \dinp{\Phi(\vec{x}^-)}{\vec{u}_{\vec{x}_*^+}}$ \Comment{$\phihat_{\vec{x}_*^+}(\vec{z}) = \sigma(m_{\vec{x}_*^+} - \dinp{\vec{z}}{\vec{u}_{\vec{x}_*^+}})$}
        \State $\mathrlap{\set{V}_{\vec{x}_*^+}}\hphantom{m_{\vec{x}_*^+}} \gets \{ \vec{x}^+ \in \set{X}^+ : \phihat_{\vec{x}_*^+}(\Phi(\vec{x}^+)) > 0 \}$
        \If{$\card{\set{V}_{\vec{x}_*^+}} > 0$} \label{line:ifVnonempty}
            \State $\set{A}^+ \gets \set{A}^+ \cup \{ \vec{x}_*^+ \}$
            \State $\mathrlap{\set{C}^+}\hphantom{\set{A}^+} \gets \set{C}^+ \setminus \set{V}_{\vec{x}_*^+}$
        \EndIf
    \EndWhile \vspace{1ex}
    \While{$\card{\bigcup_{\vec{x}_*^+ \in \set{A}^+} \set{V}_{\vec{x}_*^+}} \geq \alpha N^+$}
        \State $\mathrlap{\vec{x}_*^+}\hphantom{\set{A}^+} \gets \argmin_{\vec{x}_*^+ \in \set{A}^+} \card{ \set{V}_{\vec{x}_*^+} \setminus \bigcup_{\tilde{\vec{x}}_*^+ \in \set{A}^+ \setminus \{\vec{x}_*^+\}} \set{V}_{\tilde{\vec{x}}_*^+}}$ \label{line:pruneSelection}
        \State $\set{A}^+ \gets \set{A}^+ \setminus \{ \vec{x}_*^+ \}$
    \EndWhile
    \State $\set{A}^+ \gets \set{A}^+ \cup \{ \vec{x}_*^+ \}$ \label{line:correctExcessPruning} \Comment{$\hat{\Phi}(\vec{z}) = [\sigma(m_{\vec{x}_*^+} - \dinp{\vec{z}}{\vec{u}_{\vec{x}_*^+}})]_{\vec{x}_*^+ \in \set{A}^+}$}\vspace{1ex}
    \State \Return $F(\vec{x}) = \sign(\inp{\vec{1}}{\hat{\Phi}(\Phi(\vec{x}))})$ \Comment{$\vec{1} = (1, \ldots, 1) \in \R^{\set{A}^+}$}
\end{algorithmic}
\end{algorithm}

\subsection{Proof of the main result}
\label{sec:group-invariant-memorization-proof}

To state the generalization of Theorem~\ref{thm:size-of-translation-invariant-network}, for $\vec{x}, \vec{y} \in \R^d$, we define the pseudo-metric
\begin{equation*}
    \dist(\vec{x}, \vec{y})
    \coloneqq \min_{Q \in \set{G}} \norm{\vec{x} - Q \vec{y}}_2.
\end{equation*}
Again, we extend this notation as in the main part: for sets $\set{A}, \set{B} \subset \R^d$ and $\vec{a} \in \R^d$, we write
\begin{equation*}
    \dist(\vec{a}, \set{B}) \coloneqq \inf_{\vec{b} \in \set{B}} \dist(\vec{a}, \vec{b})
    \quad\text{and}\quad
    \dist(\set{A}, \set{B}) \coloneqq \inf_{\vec{a} \in \set{A}} \dist(\vec{a}, \set{B}).
\end{equation*}

The following definition generalizes the translation-invariant mutual covering from the main part to instead be invariant to the transformations contained in $\set{G}$. Notice that the case of $\set{G}_{\T}$ recovers Definition~\ref{def:translation-invariant-mutual-covering}.

\begin{definition}[$\set{G}$-invariant mutual covering]\label{def:group-invariant-mutual-covering}
    Recall that $\set{X} = \set{X}^+ \uplus \set{X}^-$. We call
    \begin{align*}
        \set{C}^+ = \{\vec{c}_1^+, \ldots, \vec{c}_{M^+}^+\} \subset \set{X}^+, &  \qquad r_1^+, \ldots, r_{M^+}^+ \geq 0,\\
        \set{C}^- = \{\vec{c}_1^-, \ldots, \vec{c}_{M^-}^-\} \subset \set{X}^-, &  \qquad r_1^-, \ldots, r_{M^-}^- \geq 0,
    \end{align*}
    a \emph{$\set{G}$-invariant mutual covering} of $\set{X}$, if the \emph{components} of the covering
    \begin{align*}
        \set{X}_k^+ & = \{ \vec{x}^+ \in \set{X}^+ : \dist(\vec{c}_k^+, \vec{x}^+) \leq r_k^+ \} = \bigcup_{Q \in \set{G}} \set{X}_{k, Q}^+, & \set{X}_{k, Q}^+ & \coloneqq \set{X}^+ \cap \B_2^d(Q \vec{c}_k^+, r_k^+),\\
        \set{X}_\ell^- & = \{ \vec{x}^- \in \set{X}^- : \dist(\vec{c}_\ell^-, \vec{x}^-) \leq r_\ell^- \} = \bigcup_{Q \in \set{G}} \set{X}_{\ell, Q}^-, & \set{X}_{\ell, Q}^- & \coloneqq \set{X}^- \cap \B_2^d(Q \vec{c}_\ell^-, r_\ell^-)
    \end{align*}
    cover $\set{X}^+$ and $\set{X}^-$, respectively. We call $M^+$ and $M^-$ the \emph{mutual covering numbers}.
\end{definition}

The following is the generalization of Theorem~\ref{thm:size-of-translation-invariant-network}. It states conditions under which Algorithm~\ref{alg:group-invariant-memorization} succeeds to produce an invariant network that interpolates the input data with high probability and, moreover, gives estimates on the size of this network.

\begin{theorem}[size of $\set{G}$-invariant net]\label{thm:size-of-group-invariant-network}
    Let $\set{G} \leq \set{O}(d)$ be finite with $G \coloneqq \card{\set{G}}$, let $\set{X} = \set{X}^+ \uplus \set{X}^- \subset \radius \B_2^d$ be finite. Suppose there exists a $\set{G}$-invariant mutual covering with centers satisfying $\dist(\set{C}^+, \set{C}^-) \geq \delta$ and radii satisfying
    \begin{equation*}
        r_k^+ \lesssim \frac{\dist(\vec{c}_k^+, \set{C}^-)}{\sqrt{\log(e \lambda / \dist(\vec{c}_k^+, \set{C}^-))}}
        \qquad\text{and}\qquad
        r_\ell^- \lesssim \frac{\dist(\vec{c}_\ell^-, \set{C}^+)}{\sqrt{\log(e \lambda / \dist(\vec{c}_\ell^-, \set{C}^+))}},
    \end{equation*}
    for $k \in \range{M^+}$ and $\ell \in \range{M^-}$. Set $\omega \coloneqq \max\{ \omega^+, \omega^- \}$ where
    \begin{equation*}
        \omega^+ = \max_{k \in \range{M^+}, Q \in \set{G}} \frac{w^2(\set{X}_{k, Q}^+)}{\dist^3(\vec{c}_k^+, \set{C}^-)}
        \qquad\text{and}\qquad
        \omega^- = \max_{\ell \in \range{M^-}, Q \in \set{G}} \frac{w^2(\set{X}_{\ell, Q}^-)}{\dist^3(\vec{c}_\ell^-, \set{C}^+)}.
    \end{equation*}
    Suppose $\lambda \gtrsim \radius$ and
    \begin{equation*}
        K \gtrsim \lambda \delta^{-1} \log(2 M^+ M^- G / \eta) + \lambda \omega.
    \end{equation*}
    Then, with probability at least $1 - \eta$, Algorithm~\ref{alg:group-invariant-memorization} outputs a neural network $F\colon \R^d \to \{\pm 1\}$ which is invariant under the transformations contained in $\set{G}$ and interpolates $\set{X}$, that is, for at least a fraction of $\alpha$ samples $\vec{x}^+ \in \set{X}^+$ and all samples $\vec{x}^- \in \set{X}^-$,
    \begin{equation*}
        F(\vec{x}^+) = +1
        \quad\text{and}\quad
        F(\vec{x}^-) = -1.
    \end{equation*}
    Moreover, it has $K$ and no more than $M^+$ convolutions in its first and second (generalized) convolutional layers, respectively. In particular, it has no more than $\OO(d K + M^+ G K)$ parameters.
\end{theorem}

In the remainder of this section, we prove Theorem~\ref{thm:size-of-group-invariant-network}. First, we make some simple but helpful observations. Then, we prove lower and upper bounds on quantities of interest in Lemma~\ref{lem:convolution-layer-lower-bound} and~\ref{lem:convolution-layer-upper-bound}, respectively. We connect these bounds in Lemma~\ref{lem:convolution-neuron-activation-region} which serves as the key ingredient of our proof of Theorem~\ref{thm:size-of-group-invariant-network}, stated at the very end of this section.

Let us formally define the first layer, which is constructed within the first loop of Algorithm~\ref{alg:group-invariant-memorization}. For the \emph{number of filters} $K \in \N$ and the \emph{maximal bias} $\lambda \geq 0$, let
\begin{equation*}
    \vec{g}_1, \ldots, \vec{g}_K \sim \Normal(\vec{0}, \mat{I}_d)
    \quad\text{and}\quad
    \tau_1, \ldots, \tau_K \sim \Unif([-\lambda, \lambda])
\end{equation*}
be all independent of each other. We define the \emph{random convolution layer} by
\begin{equation*}
    \Phi(\vec{x}) = [\sigma(\vec{x} \conv \vec{g}_k + \tau_k)]_{k \in \range{K}} \in \R^{K \times \set{G}},
\end{equation*}
where the addition of the biases $\tau_k$ is element-wise.

The proofs of Lemmas~\ref{lem:convolution-layer-lower-bound} and~\ref{lem:convolution-layer-upper-bound} below relate our more general setting to a setting with a fully-connected first layer. We define the \emph{random fully-connected layer $\Psi$ associated to $\Phi$} by
\begin{equation*}
    \Psi(\vec{x}) = \sigma(\mat{G} \vec{x} + \vec{\tau}),
\end{equation*}
where $\mat{G}$ has rows $\vec{g}_1, \ldots, \vec{g}_K$ and $\vec{\tau} = (\tau_1, \ldots, \tau_K)$.

The following lemma states two helpful observations which we use throughout the following proofs. The first is that transforming an input of the convolution layer is the same as shifting its output. The second is that only observing the output of the convolution layer for a fixed transformation is the same as applying the associated fully-connected layer to an appropriately transformed input.

\begin{lemma}\label{lem:convolution-layer-properties}
    For any $Q \in \set{G}$ and $\vec{x} \in \R^d$,
    \begin{equation*}
        \Phi(Q \vec{x}) = \lift{\shift}_Q [\Phi(\vec{x})]
        \quad\text{and}\quad
        [\Phi(\vec{x})]_{(:, Q)} = \Psi(Q^\top \vec{x}).
    \end{equation*}
\end{lemma}

\begin{proof}
    For $R \in \set{G}$ and $\vec{x}, \vec{g} \in \R^d$, we observe
    \begin{equation*}
        (R \vec{x}) \conv \vec{g}
        = \shift_R (\vec{x} \conv \vec{g}).
    \end{equation*}
    Indeed, for $Q \in \set{G}$,
    \begin{equation*}
        [(R \vec{x}) \conv \vec{g}]_Q
        = \inp{R \vec{x}}{Q \vec{g}}
        = \inp{\vec{x}}{R^\top Q \vec{g}}
        = [\vec{x} \conv \vec{g}]_{R^\top Q}
        = [\shift_R (\vec{x} \conv \vec{g})]_Q.
    \end{equation*}
    Let $k \in \range{K}$. Using the above and due to the element-wise application of $\tau_k$ and $\sigma$,
    \begin{equation*}
        [\Phi(Q \vec{x})]_k
        = \sigma((Q \vec{x}) \conv \vec{g}_k + \tau_k)
        = \sigma(\shift_Q(\vec{x} \conv \vec{g}_k) + \tau_k)
        = \shift_Q \sigma(\vec{x} \conv \vec{g}_k + \tau_k)
        = \shift_Q [\Phi(\vec{x})]_k,
    \end{equation*}
    which proves the first statement. For the second statement, note that
    \begin{equation*}
        [\Phi(\vec{x})]_{(k, Q)}
        = \sigma(\inp{\vec{x}}{Q \vec{g}_k} + \tau_k)
        = \sigma(\inp{Q^\top \vec{x}}{\vec{g}_k} + \tau_k)
        = [\Psi(Q^\top \vec{x})]_k.
    \end{equation*}
\end{proof}

In the case of the random fully-connected layer, \citet{dirksen2024} prove the following. To state the result, for $\set{A}, \set{B} \subset \R^d$, we denote
\begin{equation*}
    \dist_2(\set{A}, \set{B})
    \coloneqq \inf_{\vec{a} \in \set{A}, \vec{b} \in \set{B}} \norm{\vec{a} - \vec{b}}_2.
\end{equation*}

\begin{lemma}[{\citealt[Lemma~17]{dirksen2024}}]\label{lem:fully-connected-layer-lower-bound}
     Let $\set{X} = \set{X}^+ \uplus \set{X}^- \subset \radius \B_2^d$ be finite with $N^+ \coloneqq \card{\set{X}^+}$, $N^- \coloneqq \card{\set{X}^-}$. Suppose $\dist_2(\set{X}^+, \set{X}^-) \geq \delta$ and assume $\lambda \gtrsim \radius$ and
    \begin{equation*}
        K \gtrsim \delta^{-1} \lambda \log(2 N^+ N^- / \eta).
    \end{equation*}
    Then, with probability at least $1 - \eta$, for all $\vec{x}^+ \in \set{X}^+$ and $\vec{x}^- \in \set{X}^-$,
    \begin{equation*}
        \inp{\Psi(\vec{x}^+)}{\vec{v}_{\vec{x}^+}} = 0
        \qquad\text{and}\qquad
        \inp{\Psi(\vec{x}^-)}{\vec{v}_{\vec{x}^+}} \gtrsim \norm{\vec{x}^- - \vec{x}^+}_2 \lambda^{-1} K,
    \end{equation*}
    where $\vec{v}_{\vec{x}^+} = \ind{\Psi(\vec{x}^+) = \vec{0}}$.
\end{lemma}

We generalize the above result to the random convolution layer.

\begin{lemma}\label{lem:convolution-layer-lower-bound}
    Let $\set{G} \leq \set{O}(d)$ be finite with $G \coloneqq \card{\set{G}}$ and let $\set{X} = \set{X}^+ \uplus \set{X}^- \subset \radius \B_2^d$ be finite with $N^+ \coloneqq \card{\set{X}^+}$, $N^- \coloneqq \card{\set{X}^-}$. Suppose $\dist(\set{X}^+, \set{X}^-) \geq \delta$ and assume $\lambda \gtrsim \radius$ and
    \begin{equation*}
        K \gtrsim \delta^{-1} \lambda \log(2 N^+ N^- G / \eta).
    \end{equation*}
    Then, with probability at least $1 - \eta$, for all $\vec{x}^+ \in \set{X}^+$ and $\vec{x}^- \in \set{X}^-$,
    \begin{equation*}
        \dinp{\Phi(\vec{x}^+)}{\vec{u}_{\vec{x}^+}} = 0
        \qquad\text{and}\qquad
        \dinp{\Phi(\vec{x}^-)}{\vec{u}_{\vec{x}^+}} \gtrsim \dist(\vec{x}^-, \vec{x}^+) \lambda^{-1} K G,
    \end{equation*}
    where $\vec{u}_{\vec{x}^+} = \ind{\Phi(\vec{x}^+) = \vec{0}}$.
\end{lemma}

\begin{proof}
    Let $E \in \set{G}$ denote the identity element. Then, for all $\vec{x}^+ \in \set{X}^+$, by definition,
    \begin{equation*}
        0
        \leq \dinp{\Phi(\vec{x}^+)}{\vec{u}_{\vec{x}^+}}
        = \min_{R \in \set{G}} \inp{\Phi(\vec{x}^+)}{\lift{\shift}_R \vec{u}_{\vec{x}^+}}
        \leq \inp{\Phi(\vec{x}^+)}{\lift{\shift}_E \vec{u}_{\vec{x}^+}}
        = \inp{\Phi(\vec{x}^+)}{\vec{u}_{\vec{x}^+}}
        = 0.
    \end{equation*}
    This proves the first claim. For the second claim, consider the event $\event{A}$ where, for all $\vec{x}^+ \in \set{X}^+$, $\vec{x}^- \in \set{X}^-$ and $R, Q \in \set{G}$,
    \begin{equation*}
        \inp{\Psi(Q^\top \vec{x}^-)}{\vec{v}_{Q^\top R \vec{x}^+}}
        \gtrsim \norm{Q^\top \vec{x}^- - Q^\top R \vec{x}^+}_2 \lambda^{-1} K
        = \norm{\vec{x}^- - R \vec{x}^+}_2 \lambda^{-1} K.
    \end{equation*}
    By Lemma~\ref{lem:fully-connected-layer-lower-bound} applied to the sets $Q^\top \set{X}^-$ and $Q^\top R \set{X}^+$, followed by a union bound over $Q, R \in \set{G}$, we obtain $\Pr(\event{A}) \geq 1 - \eta$ under our assumptions.

    On $\event{A}$, for all $\vec{x}^+ \in \set{X}^+, \vec{x}^- \in \set{X}^-$ and $R\in \mathcal{G}$,
    \begin{align*}
        \inp{\Phi(\vec{x}^-)}{\vec{u}_{R \vec{x}^+}}
        = {} & \sum_{Q \in \set{G}} \inp{[\Phi(\vec{x}^-)]_{(:, Q)}}{[\vec{u}_{R\vec{x}^+}]_{(:, Q)}}\\
        = {} & \sum_{Q \in \set{G}} \inp{\Psi(Q^\top \vec{x}^-)}{\vec{v}_{Q^\top R \vec{x}^+}}
        \gtrsim \norm{\vec{x}^- - R \vec{x}^+}_2 \lambda^{-1} K G,
    \end{align*}
    where we used Lemma~\ref{lem:convolution-layer-properties} and, in particular,
    \begin{equation*}
        [\vec{u}_{R \vec{x}^+}]_{(:, Q)}
        = \ind{[\Phi(R \vec{x}^+)]_{(:, Q)}
        = \vec{0}} = \ind{\Psi(Q^\top R \vec{x}^+)
        = \vec{0}} = \vec{v}_{Q^\top R \vec{x}^+}.
    \end{equation*}
    This immediately implies, for all $\vec{x}^+ \in \set{X}^+$ and $\vec{x}^- \in \set{X}^-$,
    \begin{align*}
        \dinp{\Phi(\vec{x}^-)}{\vec{u}_{\vec{x}^+}}
        = {} &\min_{R \in \set{G}} \inp{\Phi(\vec{x}^-)}{\lift{\shift}_R \vec{u}_{\vec{x}^+}}\\
        = {} &\min_{R \in \set{G}} \inp{\Phi(\vec{x}^-)}{\vec{u}_{R \vec{x}^+}}
        \gtrsim \dist(\vec{x}^-, \vec{x}^+) \lambda^{-1} K G,
    \end{align*}
    where, again, we used Lemma~\ref{lem:convolution-layer-properties} and, in particular,
    \begin{equation*}
        \lift{\shift}_R \vec{u}_{\vec{x}^+}
        = \ind{\lift{\shift}_R \Phi(\vec{x}^+) = \vec{0}}
        = \ind{\Phi(R \vec{x}^+) = \vec{0}}
        = \vec{u}_{R \vec{x}^+}.
    \end{equation*}
\end{proof}

The following observation again considers a random fully-connected layer and was stated by \citet[Lemma~18]{dirksen2024} as a direct consequence of Theorem~2.9 of \citet{dirksen2021}.\footnote{The formulation in \citet{dirksen2024} uses the Gaussian complexity, which we replaced by the Gaussian width. This is possible since, in their version, the components are centered and thus contain the origin. We may then use, \eg, Exercise~7.6.9 of \citet{vershynin2018high}.}

\begin{lemma}[{\citealt[Lemma~18]{dirksen2024}}]\label{lem:fully-connected-layer-upper-bound}
    Let $\vec{c}_1,\ldots,\vec{c}_M\subset \R^d$, $s_1, \ldots, s_M \geq 0$ and
    \begin{equation*}
        r_m \lesssim \frac{s_m}{\sqrt{\log(e \lambda / s_m)}},
        \qquad s = \min_{m \in \range{M}} s_m.
    \end{equation*}
    Consider $\set{X}_1, \ldots, \set{X}_M \subset \R^d$ such that, for all $m \in \range{M}$,
    \begin{equation*}
        \set{X}_m \subset \B_2^d(\vec{c}_m, r_m) \subset \radius\B_2^d.
    \end{equation*}
    Suppose $\lambda \gtrsim \radius$ and
    \begin{align*}
        K \gtrsim s^{-1} \lambda\log (2 M / \eta) + \lambda\max_{m \in [M]} s_m^{-3} w^2(\set{X}_m).
    \end{align*}
    Then, with probability at least $1 - \eta$, for all $m \in [M]$ and $\vec{x} \in \set{X}_m$,
    \begin{equation*}
        \norm{\Psi(\vec{c}_m) - \Psi(\vec{x})}_1
        \lesssim s_m \lambda^{-1} K.
    \end{equation*}
\end{lemma}

We derive an analogue for the random convolution layer.

\begin{lemma}\label{lem:convolution-layer-upper-bound}
    Let $\set{G} \leq \set{O}(d)$ be finite with $G \coloneqq \card{\set{G}}$. Let $\vec{c}_1, \ldots, \vec{c}_M \in \R^d$, $s_1, \ldots, s_M \geq 0$ and
    \begin{equation*}
        r_m \lesssim \frac{s_m}{\sqrt{\log(e \lambda/s_m)}},
        \qquad s = \min_{m \in \range{M}} s_m.
    \end{equation*}
    Consider $\set{X}_1, \ldots, \set{X}_M \subset \R^d$ such that, for all $m \in \range{M}$,
    \begin{equation}
    \label{eqnXmBallCond}
        \set{X}_m \subset \{ \vec{x} \in \R^d : \dist(\vec{c}_m, \vec{x}) \leq r_m \} \subset \radius\B_2^d.
    \end{equation}
    Suppose $\lambda \gtrsim \radius$ and, for $\set{X}_{m, Q} \coloneqq \set{X}_m \cap \B_2^d(Q \vec{c}_m, r_m)$,
    \begin{align*}
        K \gtrsim s^{-1} \lambda\log (2 M G / \eta) + \lambda\max_{m \in [M], Q \in \set{G}} s_m^{-3} w^2(\set{X}_{m, Q}).
    \end{align*}
    Then, with probability at least $1 - \eta$, for all $m \in [M]$ and $\vec{x} \in \set{X}_m$,
    \begin{equation*}
        \min_{Q \in \set{G}} \norm{\lift{\shift}_Q \Phi(\vec{c}_m) - \Phi(\vec{x})}_1
        \lesssim s_m \lambda^{-1} K G.
    \end{equation*}
\end{lemma}

\begin{proof}
    Let $\event{A}$ denote the event where, for all $m \in \range{M}$, $Q, R \in \set{G}$ and $\vec{x} \in \set{X}_{m, Q}$,
    \begin{equation*}
        \norm{\Psi(R^\top Q \vec{c}_m) - \Psi(R^\top \vec{x})}_1
        \lesssim s_m \lambda^{-1} K.
    \end{equation*}
    By Lemma~\ref{lem:fully-connected-layer-upper-bound}, applied to the sets $R^\top \set{X}_{m, Q} \subset \B_2^d(R^\top Q \vec{c}_m, r_m) \subset \radius \B_2^d$ for $m \in \range{M}$ and $R, Q \in \set{G}$, $\Pr(\event{A}) \geq 1 - \eta$ under our assumptions. Here, we used that $\gw(R^\top \set{X}_{m, Q})=\gw(\set{X}_{m, Q})$ by rotation invariance of the Gaussian width.

    Let $m \in \range{M}$. On $\event{A}$, for $Q \in \set{G}$ and $\vec{x} \in \set{X}_{m, Q}$,
    \begin{align*}
        \norm{\Phi(Q \vec{c}_m) - \Phi(\vec{x})}_1
        & = \sum_{R \in \set{G}} \norm{[\Phi(Q \vec{c}_m)]_{(:, R)} - [\Phi(\vec{x})]_{(:, R)}}_1\\
        & = \sum_{R \in \set{G}} \norm{\Psi(R^\top Q \vec{c}_m) - \Psi(R^\top \vec{x})}_1
        \lesssim s_m \lambda^{-1} K G,
    \end{align*}
    where we used Lemma~\ref{lem:convolution-layer-properties}. By \eqref{eqnXmBallCond}, for $\vec{x} \in \set{X}_m$ there exists $Q^* \in \set{G}$ with $\vec{x} \in \set{X}_{m, Q^*}$ and thus
    \begin{align*}
        \min_{Q \in \set{G}} \norm{\lift{\shift}_Q \Phi(\vec{c}_m) - \Phi(\vec{x})}_1
        & \leq \norm{\lift{\shift}_{Q^*} \Phi(\vec{c}_m) - \Phi(\vec{x})}_1\\
        & = \norm{\Phi(Q^* \vec{c}_m) - \Phi(\vec{x})}_1
        \lesssim s_m \lambda^{-1} K G.
    \end{align*}
\end{proof}

The following lemma is the key ingredient for the proof of Theorem~\ref{thm:size-of-group-invariant-network}. It states that under our assumptions, with high probability, a neuron
\begin{equation*}
    \phihat_{\vec{x}_*^+}(\vec{z}) = \sigma(m_{\vec{x}_*^+} - \dinp{\vec{z}}{\vec{u}_{\vec{x}_*^+}}),
    \quad\text{where}\quad
    m_{\vec{x}_*^+} = \min_{\vec{x}^- \in \set{X}^-} \dinp{\Phi(\vec{x}^-)}{\vec{u}_{\vec{x}_*^+}},
\end{equation*}
associated to a sample $\vec{x}_*^+ \in \set{X}^+$, activates on the (image of the) whole component that this sample resides in. At the same time, it does not activate on (the image of) any sample of the other class. This lemma is the analogue to Lemma~19 from \citet{dirksen2024}.

\begin{lemma}\label{lem:convolution-neuron-activation-region}
   Under the assumptions of Theorem~\ref{thm:size-of-group-invariant-network}, with probability at least $1 - \eta$, for any $k \in \range{M^+}$ and $\vec{x}_*^+ \in \set{X}_k^+$,
    \begin{equation*}
        \phihat_{\vec{x}_*^+}(\Phi(\vec{x}^+)) > 0 \quad \text{for all $ \vec{x}^+ \in \set{X}_k^+$}
        \quad\text{and}\quad
        \phihat_{\vec{x}_*^+}(\Phi(\vec{x}^-)) = 0 \quad \text{for all $ \vec{x}^- \in \set{X}^-$}.
    \end{equation*}
\end{lemma}

\begin{proof}
    Clearly, the choice of $m_{\vec{x}_*^+}$ ensures the second property. For the first property, it remains to show that, with probability at least $1 - \eta$, for any $k \in \range{M^+}$ and all $\vec{x}_*^+, \vec{x}^+ \in \set{X}_k^+$,
    \begin{equation*}
        \dinp{\Phi(\vec{x}^+)}{\vec{u}_{\vec{x}_*^+}}
        < m_{\vec{x}_*^+}
        = \min_{\vec{x}^- \in \set{X}^-} \dinp{\Phi(\vec{x}^-)}{\vec{u}_{\vec{x}_*^+}}.
    \end{equation*}

    Let $\event{A}$ denote the event where, for $k \in \range{M^+}$ and $\ell \in \range{M^-}$,
    \begin{equation*}
        \dinp{\Phi(\vec{c}_\ell^-)}{\vec{u}_{\vec{c}_k^+}}
        \geq c_1 \cdot \dist(\vec{c}_\ell^-, \vec{c}_k^+) \lambda^{-1} K G.
    \end{equation*}
    By Lemma~\ref{lem:convolution-layer-lower-bound} applied to $\set{C}^+$ and $\set{C}^-$, $\Pr(\event{A}) \geq 1 - \eta / 2$ under our assumptions.
    
    Let $\event{B}$ denote the event where, for $k \in \range{M^+}$ and $\vec{x}^+ \in \set{X}_k^+$,
    \begin{equation*}
        \min_{Q \in \set{G}} \norm{\Phi(\vec{x}^+) - \lift{\shift}_Q \Phi(\vec{c}_k^+)}_1
        \leq C_2 \cdot s_k^+ \lambda^{-1} K G
    \end{equation*}
    and, for $\ell \in \range{M^-}$ and $\vec{x}^- \in \set{X}_\ell^-$,
    \begin{equation*}
        \min_{Q \in \set{G}} \norm{\Phi(\vec{x}^-) - \lift{\shift}_Q \Phi(\vec{c}_\ell^-)}_1
        \leq C_2 \cdot s_\ell^- \lambda^{-1} K G,
    \end{equation*}
    where
    \begin{equation*}
        s_k^+ \coloneqq \frac{c_1}{12 C_2} \dist(\vec{c}_k^+, \set{C}^-)
        \qquad\text{and}\qquad
        s_\ell^- \coloneqq \frac{c_1}{4 C_2} \dist(\vec{c}_\ell^-, \set{C}^+).
    \end{equation*}
    By Lemma~\ref{lem:convolution-layer-upper-bound} and a union bound, $\Pr(\event{B}) \geq 1 - \eta / 2$ under our assumptions. Hence, $\Pr(\event{A} \cap \event{B}) \geq 1 - \eta$ and for the remainder of the proof we condition on this event.

    Let $k \in \range{M^+}$ and $\vec{x}_*^+, \vec{x}^+ \in \set{X}_k^+$. For $Q, R \in \set{G}$,
    \begin{align*}
        \dinp{\Phi(\vec{x}^+)}{\vec{u}_{\vec{x}_*^+}}
        \leq {} & \inp{\Phi(\vec{x}^+)}{\lift{\shift}_Q \vec{u}_{\vec{x}_*^+}}\\
        = {} & \inp{\Phi(\vec{x}^+) - \lift{\shift}_Q \Phi(\vec{x}_*^+)}{\lift{\shift}_Q \vec{u}_{\vec{x}_*^+}}\\
        \leq {} & \norm{\Phi(\vec{x}^+) - \lift{\shift}_Q \Phi(\vec{x}_*^+)}_1\\
        \leq {} & \norm{\Phi(\vec{x}^+) - \lift{\shift}_R \Phi(\vec{c}_k^+)}_1 + \norm{\lift{\shift}_R \Phi(\vec{c}_k^+) - \lift{\shift}_Q \Phi(\vec{x}_*^+)}_1\\
        = {} & \norm{\Phi(\vec{x}^+) - \lift{\shift}_R \Phi(\vec{c}_k^+)}_1 + \norm{\lift{\shift}_{Q^\top R} \Phi(\vec{c}_k^+) - \Phi(\vec{x}_*^+)}_1,
    \end{align*}
    where we used $\inp{\lift{\shift}_Q \Phi(\vec{x}_*^+)}{\lift{\shift}_Q \vec{u}_{\vec{x}_*^+}} = \inp{\Phi(\vec{x}_*^+)}{\vec{u}_{\vec{x}_*^+}} = 0$, Hölder's inequality and \eqref{eqn:elemShiftProp}. Minimizing successively over $Q$ and $R$ yields
    \begin{equation*}
        \dinp{\Phi(\vec{x}^+)}{\vec{u}_{\vec{x}_*^+}}
        \leq \min_{Q \in \set{G}} \norm{\Phi(\vec{x}^+) - \lift{\shift}_Q \Phi(\vec{c}_k^+)}_1 + \min_{Q \in \set{G}} \norm{\Phi(\vec{x}_*^+) - \lift{\shift}_Q \Phi(\vec{c}_k^+)}_1.
    \end{equation*}
    Using $\event{B}$, we find the upper bound
    \begin{equation}\label{eqn:mainLemmaUpBd}
        \dinp{\Phi(\vec{x}^+)}{\vec{u}_{\vec{x}_*^+}}
        \leq 2 C_2 \cdot s_k^+ \lambda^{-1} K G.
    \end{equation}
    Let $\vec{x}^- \in \set{X}^-$ and pick $\ell \in \range{M^-}$ such that $\vec{x}^- \in \set{X}_\ell^-$. For $Q, R, U \in \set{G}$,
    \begin{align*}
        & \inp{\Phi(\vec{x}^-)}{\lift{\shift}_Q \vec{u}_{\vec{x}_*^+}}\\ 
        & \qquad = \inp{\lift{\shift}_U \Phi(\vec{c}_\ell^-)}{\lift{\shift}_{Q R} \vec{u}_{\vec{c}_k^+}} + \inp{\lift{\shift}_Q \vec{u}_{\vec{x}_*^+} - \lift{\shift}_{Q R} \vec{u}_{\vec{c}_k^+}}{\lift{\shift}_U \Phi(\vec{c}_\ell^-)} + \inp{\lift{\shift}_Q \vec{u}_{\vec{x}_*^+}}{\Phi(\vec{x}^-) - \lift{\shift}_U \Phi(\vec{c}_\ell^-)}\\
        & \qquad \geq \inp{\Phi(\vec{c}_\ell^-)}{\lift{\shift}_{U^\top Q R} \vec{u}_{\vec{c}_k^+}} - \norm{\lift{\shift}_Q \Phi(\vec{x}_*^+) - \lift{\shift}_{Q R} \Phi(\vec{c}_k^+)}_1 - \norm{\Phi(\vec{x}^-) - \lift{\shift}_U \Phi(\vec{c}_\ell^-)}_1\\
        & \qquad \geq \dinp{\Phi(\vec{c}_\ell^-)}{\vec{u}_{\vec{c}_k^+}} - \norm{\Phi(\vec{x}_*^+) - \lift{\shift}_R \Phi(\vec{c}_k^+)}_1 - \norm{\Phi(\vec{x}^-) - \lift{\shift}_U \Phi(\vec{c}_\ell^-)}_1,
    \end{align*}
    where, again, we used Hölder's inequality and the properties \eqref{eqn:elemShiftProp} of (lifted) circular shift operators. Minimizing over $Q$ and maximizing over $R$ and $U$ yields
    \begin{equation*}
        \dinp{\Phi(\vec{x}^-)}{\vec{u}_{\vec{x}_*^+}}
        \geq \dinp{\Phi(\vec{c}_\ell^-)}{\vec{u}_{\vec{c}_k^+}} - \min_{Q \in \set{G}} \norm{\Phi(\vec{x}_*^+) - \lift{\shift}_Q \Phi(\vec{c}_k^+)}_1 - \min_{Q \in \set{G}} \norm{\Phi(\vec{x}^-) - \lift{\shift}_Q \Phi(\vec{c}_\ell^-)}_1.
    \end{equation*}
    Using both $\event{A}$ and $\event{B}$, we find the lower bound
    \begin{equation}\label{eqn:mainLemmaLowBd}
        \dinp{\Phi(\vec{x}^-)}{\vec{u}_{\vec{x}_*^+}}
        \geq c_1 \cdot \dist(\vec{c}_\ell^-, \vec{c}_k^+) \lambda^{-1} K G - C_2 \cdot s_k^+ \lambda^{-1} K G - C_2 \cdot s_\ell^- \lambda^{-1} K G.
    \end{equation}
    Combining the upper and lower bounds \eqref{eqn:mainLemmaUpBd} and \eqref{eqn:mainLemmaLowBd}, we find
    \begin{equation*}
        \dinp{\Phi(\vec{x}^+)}{\vec{u}_{\vec{x}_*^+}} < \dinp{\Phi(\vec{x}^-)}{\vec{u}_{\vec{x}_*^+}}
    \end{equation*}
    for all $k \in \range{M^+}$, $\vec{x}_*^+, \vec{x}^+ \in \set{X}_k^+$ and $\vec{x}^- \in \set{X}^-$, where we used that
    \begin{equation*}
        s_k^+ < \frac{c_1}{6 C_2} \cdot \dist(\vec{c}_k^+, \set{C}^-)
        \quad\text{and}\quad
        s_\ell^- < \frac{c_1}{2 C_2} \cdot \dist(\vec{c}_\ell^-, \set{C}^+).
    \end{equation*}
\end{proof}

\begin{proof}[Theorem~\ref{thm:size-of-group-invariant-network}]
    We use the notation from Algorithm~\ref{alg:group-invariant-memorization}. In the following we condition on the event of Lemma~\ref{lem:convolution-neuron-activation-region}, which occurs with probability at least $1 - \eta$ under the stated assumptions. 
    
    \paragraph{Termination and number of neurons.} The first loop constructs the random convolution layer
    \begin{equation*}
        \Phi(\vec{x}) = [\sigma(\vec{x} \conv \vec{g}_k + \tau_k)]_{k \in \range{K}}.
    \end{equation*}
    
    For the second loop we claim that, after the $i$-th iteration, $\set{A}^+ = \{ \vec{x}_1^+, \ldots, \vec{x}_i^+ \}$ and
    \begin{equation*}
        \set{C}^+ = \set{X}^+ \setminus (\set{V}_{\vec{x}_1^+} \cup \cdots \cup \set{V}_{\vec{x}_i^+}) \subset \set{X}^+ \setminus (\set{X}_{k_1}^+ \cup \cdots \cup \set{X}_{k_i}^+),
    \end{equation*}
    where each $\vec{x}_j^+$ is part of some component $\set{X}_{k_j}^+$ for pairwise distinct $k_1, \ldots, k_j \in \range{M^+}$.
    
    For $i = 0$, we have $\set{A}^+ = \emptyset$ and $\set{C}^+ = \set{X}^+$, and hence the claim is trivially satisfied. Suppose that the claim holds for the $i$-th iteration. In the $(i + 1)$-th iteration, the algorithm selects a random candidate sample $\vec{x}_{i+1}^+ \in \set{C}^+$ which is part of at least one new component $\set{X}_{k_{i+1}}$ for some $k_{i+1} \in \range{M^+} \setminus \{ k_1, \ldots, k_i \}$. By Lemma~\ref{lem:convolution-neuron-activation-region}, we have $\emptyset \neq \set{X}_{k_{i+1}}^+ \subset \set{V}_{\vec{x}_{i+1}^+}$ and the condition of the if-statement in line~\ref{line:ifVnonempty} is true. Hence, $\set{A}^+ = \{ \vec{x}_1^+, \ldots, \vec{x}_{i+1}^+ \}$ and
    \begin{equation*}
        \set{C}^+ = \set{X}^+ \setminus (\set{V}_{\vec{x}_1^+} \cup \cdots \cup \set{V}_{\vec{x}_{i+1}^+}) \subset \set{X}^+ \setminus (\set{X}_{k_1}^+ \cup \cdots \cup \set{X}_{k_{i+1}}^+),
    \end{equation*}
    and the claim follows by induction. In particular, after at most $M^+$ steps, the loop terminates due to $\set{C}^+ = \emptyset$ and $\set{A}^+$ contains no more than $M^+$ samples. Moreover, $\set{X}^+ = \bigcup_{\vec{x}_*^+ \in \set{A}^+} \set{V}_{\vec{x}_*^+}$.

    Finally, we examine the last loop. As long as 
    \begin{equation}\label{eqn:pruneCond}
        \Bigg|\bigcup_{\vec{x}_*^+ \in \set{A}^+} \set{V}_{\vec{x}_*^+}\Bigg| \geq \alpha N^+,
    \end{equation}
    we remove samples from $\set{A}^+$. The selection rule in line~\ref{line:pruneSelection} is not important here. If removing a sample leads to a violation of this condition, the loop terminates and we add the last removed sample back to $\set{A}^+$ in line~\ref{line:correctExcessPruning} so that the condition \eqref{eqn:pruneCond} is satisfied again.

    \paragraph{Interpolation.} By Lemma~\ref{lem:convolution-neuron-activation-region}, for all $\vec{x}^- \in \set{X}^-$, we have $\hat{\Phi}(\Phi(\vec{x}^-)) = \vec{0}$ and hence $F(\vec{x}^-) = -1$. On the other hand, consider $\vec{x}^+ \in \bigcup_{\vec{x}_*^+ \in \set{A}^+} \set{V}_{\vec{x}_*^+}$, which includes at least a fraction of $\alpha$ of the samples from $\set{X}^+$ as we recall from above. Then, there exists $\vec{x}_*^+ \in \set{A}^+$ with $\vec{x}^+ \in \set{V}_{\vec{x}_*^+}$, and hence $[\hat{\Phi}(\Phi(\vec{x}^+))]_{\vec{x}_*^+} > 0$. Consequently, $F(\vec{x}^+) = +1$.

    \paragraph{Invariance.} Let $Q \in \set{G}$. Observe, for $\vec{z}, \vec{u} \in \R^{K \times \set{G}}$,
    \begin{equation*}
        \dinp{\lift{\shift}_Q \vec{z}}{\vec{u}}
        = \min_{R \in \set{G}} \inp{\lift{\shift}_Q \vec{z}}{\lift{\shift}_R \vec{u}}
        = \min_{R \in \set{G}} \inp{\vec{z}}{\lift{\shift}_{Q^\top R} \vec{u}}
        = \min_{R \in \set{G}} \inp{\vec{z}}{\lift{\shift}_{R} \vec{u}}
        = \dinp{\vec{z}}{\vec{u}},
    \end{equation*}
    and hence $\hat{\Phi}(\lift{\shift}_Q \vec{z}) = \hat{\Phi}(\vec{z})$. Combining this with Lemma~\ref{lem:convolution-layer-properties}, we find for any $\vec{x} \in \R^d$ that
    \begin{equation*}
        F(Q \vec{x})
        = \sign(\inp{\vec{1}}{\hat{\Phi}(\Phi(Q \vec{x}))})
        = \sign(\inp{\vec{1}}{\hat{\Phi}(\lift{\shift}_Q \Phi(\vec{x}))})
        = \sign(\inp{\vec{1}}{\hat{\Phi}(\Phi(\vec{x}))})
        = F(\vec{x}).
    \end{equation*}

    \paragraph{Number of parameters.} The first layer $\Phi$ has weights $\vec{g}_k \in \R^d$ and biases $\tau_k \in \R$ for $k \in \range{K}$, which is $(d + 1) K$ parameters. The second layer $\hat{\Phi}$ has weights $\vec{u}_{\vec{x}_*^+} \in \R^{K \times \set{G}}$ and biases $m_{\vec{x}_*^+} \in \R$ for $\vec{x}_*^+ \in \set{A}^+$ with $\card{\set{A}^+} \leq M^+$, which is at most $(K G + 1) M^+$ parameters. The last, linear layer, has weights $\vec{1} \in \R^{\set{A}^+}$, which is at most $M^+$ parameters. Hence, the total is no more than
    \begin{equation*}
        (d + 1) K + (K G + 1) M^+ + M^+
        \lesssim (d + M^+ G) K.
    \end{equation*}
\end{proof}

\section{Multi-class interpolation algorithm}
\label{sec:multiclass-algorithm}

In Appendix~\ref{sec:group-invariant-memorization} we introduced Algorithm~\ref{alg:group-invariant-memorization} in the context of a binary classification problem. Its approach can be easily generalized to multiple classes, see Algorithm~\ref{alg:group-invariant-memorization-multiclass}. In this section, we outline how this generalization works but omit a full proof as it would be very similar to the proof of Theorem~\ref{thm:size-of-group-invariant-network}. The approach is similar to the one of \citet{dirksen2024}.

A common approach to extend a binary classification method to multiple classes is the one-vs-rest method, where for each class a binary classifier is trained to distinguish it from all other classes. At inference time, the class label corresponding to the classifier with the highest confidence score is returned. We extend Algorithm~\ref{alg:group-invariant-memorization} in a similar, yet not exactly the same manner. In the first loop (lines~\ref{alg-line:group-invariant-memorization-multiclass-begin-first-layer} to~\ref{alg-line:group-invariant-memorization-multiclass-end-first-layer}) we sample weights and biases of a random convolution layer which will be \emph{reused} for each one-vs-rest subproblem. In the second loop (lines~\ref{alg-line:group-invariant-memorization-begin-one-vs-rest} to~\ref{alg-line:group-invariant-memorization-end-one-vs-rest}) we solve each of the one-vs-rest subproblems, utilizing a procedure (lines~\ref{alg-line:group-invarian-memorization-multiclass-begin-procedure} to~\ref{alg-line:group-invarian-memorization-multiclass-end-procedure}) which is the same as the second loop in Algorithm~\ref{alg:group-invariant-memorization}. Already at this point, one can use a simple union bound to prove a high success probability for the algorithm. What remains is the modified pruning loop (lines~\ref{alg-line:group-invariant-memorization-begin-pruning} to~\ref{alg-line:group-invariant-memorization-end-pruning}) which works similarly to before. However, instead of selecting the neuron with the least contribution for one class alone, we consider neurons over all classes. That at least a fraction of $\alpha$ of the samples of $\set{X}$ are interpolated after this step can be derived analogously to the corresponding part of the proof of Theorem~\ref{thm:size-of-group-invariant-network}. In the end, we return the network (line~\ref{alg-line:group-invariant-memorization-return}), which reuses the first random layer and stacks the second layer for each of the one-vs-rest subproblems.

\begin{algorithm}
\caption{$\set{G}$-invariant multi-class memorization}\label{alg:group-invariant-memorization-multiclass}
\begin{algorithmic}[1]
    \Require{$\set{G} \leq \set{O}(d)$ finite, $\set{X} = \set{X}^{(1)} \uplus \cdots \uplus \set{X}^{(L)} \subset \R^d$ finite with $N \coloneqq \card{\set{X}}$, number of random filters $K \in \N$, maximal bias $\lambda \geq 0$, interpolation fraction $\alpha \in (0, 1]$}
    \Ensure{$\set{G}$-invariant (generalized) convolutional neural net $F\colon \R^d \to \{0, 1\}^L$}
    \For{$k \in \range{K}$}\label{alg-line:group-invariant-memorization-multiclass-begin-first-layer}
        \State $\vec{g}_k \gets \text{independent sample from } \Normal(\vec{0}, \mat{I}_d)$
        \State $\mathrlap{\tau_k}\hphantom{\vec{g}_k} \gets \text{independent sample from } \Unif([-\lambda, \lambda])$
    \EndFor \Comment{$\vec{g} = (\vec{g}_1, \ldots, \vec{g}_K), \vec{\tau} = (\tau_1, \ldots, \tau_K)$} \label{alg-line:group-invariant-memorization-multiclass-end-first-layer}\vspace{1ex}
    \For{$\ell \in \range{L}$} \label{alg-line:group-invariant-memorization-begin-one-vs-rest}
        \State $\set{A}^{(\ell)} \gets \Call{Interpolate}{\set{X}^{(\ell)}, \set{X} \setminus \set{X}^{(\ell)}, \vec{g}, \vec{\tau}}$
    \EndFor \Comment{$\set{A} = \set{A}^{(1)} \cup \cdots \cup \set{A}^{(L)}$} \label{alg-line:group-invariant-memorization-end-one-vs-rest} \vspace{1ex}
    \While{$\card{\bigcup_{\vec{x}_* \in \set{A}} \set{V}_{\vec{x}_*}} \geq \alpha N$} \label{alg-line:group-invariant-memorization-begin-pruning}
        \State $\ell, \vec{x}_* \gets \argmin_{\ell \in \range{L}, \vec{x}_* \in \set{A}^{(\ell)}} \card{ \set{V}_{\vec{x}_*} \setminus \bigcup_{\tilde{\vec{x}}_* \in \set{A} \setminus \{\vec{x}_*\}} \set{V}_{\tilde{\vec{x}}_*}}$
        \State $\set{A}^{(\ell)} \gets \set{A}^{(\ell)} \setminus \{ \vec{x}_* \}$
    \EndWhile
    \State $\set{A}^{(\ell)} \gets \set{A}^{(\ell)} \cup \{ \vec{x}_* \}$ \label{alg-line:group-invariant-memorization-end-pruning} \Comment{$\hat{\Phi}^{(\ell)}(\vec{z}) = [\sigma(m_{\vec{x}_*} - \dinp{\vec{z}}{\vec{u}_{\vec{x}_*}})]_{\vec{x}_* \in \set{A}^{(\ell)}}$} \vspace{1ex}
    \State \Return $F(\vec{x}) = (F^{(1)}(\vec{x}), \ldots, F^{(L)}(\vec{x}))$ \Comment{$F^{(\ell)}(\vec{x}) = \sigma(\inp{\vec{1}}{\hat{\Phi}^{(\ell)}(\Phi(\vec{x}))})$} \label{alg-line:group-invariant-memorization-return} \vspace{2ex}
    \Procedure{Interpolate}{$\set{X}^+$, $\set{X}^-$, $\vec{g}$, $\vec{\tau}$}  \label{alg-line:group-invarian-memorization-multiclass-begin-procedure}
        \State $\set{A}^+ \gets \emptyset$ \Comment{$\Phi(\vec{x}) = [\sigma(\vec{x} \conv \vec{g}_k + \tau_k)]_{k \in \range{K}}$}
        \State $\mathrlap{\set{C}^+}\hphantom{\set{A}^+} \gets \set{X}^+$
        \While{$\card{\set{C}^+} > 0$}
            \State $\vec{x}_*^+ \gets \text{random sample from } \set{C}^+$
            \State $\mathrlap{\set{C}^+}\hphantom{\vec{x}_*^+} \gets \set{C}^+ \setminus \{ \vec{x}_*^+ \}$
            \State $\mathrlap{\vec{u}_{\vec{x}_*^+}}\hphantom{m_{\vec{x}_*^+}} \gets \ind{\Phi(\vec{x}_*^+) = \vec{0}}$
            \State $m_{\vec{x}_*^+} \gets \min_{\vec{x}^- \in \set{X}^-} \dinp{\Phi(\vec{x}^-)}{\vec{u}_{\vec{x}_*^+}}$ \Comment{$\phihat_{\vec{x}_*^+}(\vec{z}) = \sigma(m_{\vec{x}_*^+} - \dinp{\vec{z}}{\vec{u}_{\vec{x}_*^+}})$}
            \State $\mathrlap{\set{V}_{\vec{x}_*^+}}\hphantom{m_{\vec{x}_*^+}} \gets \{ \vec{x}^+ \in \set{X}^+ : \phihat_{\vec{x}_*^+}(\Phi(\vec{x}^+)) > 0 \}$
            \If{$\card{\set{V}_{\vec{x}_*^+}} > 0$}
                \State $\set{A}^+ \gets \set{A}^+ \cup \{ \vec{x}_*^+ \}$
                \State $\mathrlap{\set{C}^+}\hphantom{\set{A}^+} \gets \set{C}^+ \setminus \set{V}_{\vec{x}_*^+}$
            \EndIf
        \EndWhile\vspace{1ex}
        \State \Return $\set{A}^+$
    \EndProcedure \label{alg-line:group-invarian-memorization-multiclass-end-procedure}
\end{algorithmic}
\end{algorithm}

\newpage

\bibliography{ref}

\end{document}